\documentclass[moor]{informs1}              

\usepackage{natbib}
 \NatBibNumeric
 \def\bibfont{\small}%
 \def\bibsep{\smallskipamount}%
 \def\bibhang{24pt}%
 \bibpunct[, ]{[}{]}{,}{n}{}{,}%
 
 \usepackage{times}
\usepackage{amsmath}
\usepackage{algpseudocode}
\usepackage{algorithm}
\usepackage{bbm}
\usepackage{thmtools}
\usepackage{thm-restate}
\usepackage{enumitem}
\usepackage{lipsum}

\usepackage[colorlinks=true,breaklinks=true,bookmarks=true,urlcolor=blue,
     citecolor=blue,linkcolor=blue,bookmarksopen=false,draft=false]{hyperref}
\usepackage{cleveref}

\newcommand{\banditMNL}{MNL-Bandit}
\newcommand{\reg}{\mbox{Reg}}
\newcommand{\Reg}{\reg}
\newcommand{\Ex}{\bb{E}}

\newtheorem{theor1}{Theorem}
\newtheorem{lemma1}{Lemma}

\newtheorem{assumption1}{Assumption}

\def\ep#1{\mathcal{#1}}
\def\Halmos{\square}
\def\mb#1{\mathbf{#1}}
\def\bb#1{\mathbb{#1}}
\def\bbm#1{\mathbbm{#1}}
\newcommand{\suchthat}{\;\ifnum\currentgrouptype=16 \middle\fi|\;}

\graphicspath{{plots/}}

\newcommand{\comment}[1]{}

\def\EMAIL#1{\href{mailto:#1}{#1}}

\EquationsNumberedThrough    

\begin{document}


 \RUNAUTHOR{Agrawal et al.} 

\RUNTITLE{Thompson Sampling for the MNL-Bandit}

 \TITLE{Thompson Sampling for the MNL-Bandit}

\ARTICLEAUTHORS{%
\AUTHOR{Shipra Agrawal}
\AFF{Industrial Engineering and Operations Research, Columbia University, New York, NY. \EMAIL{sa3305@columbia.edu}}
\AUTHOR{Vashist Avadhanula}
\AFF{Decision, Risk and Operations, Columbia Business School, New York, NY. \EMAIL{vavadhanula18@gsb.columbia.edu}}
\AUTHOR{Vineet Goyal}
\AFF{Industrial Engineering and Operations Research, Columbia University, New York, NY. \EMAIL{vg2277@columbia.edu}}
\AUTHOR{Assaf Zeevi}
\AFF{Decision, Risk and Operations, Columbia Business School, New York, NY. \EMAIL{assaf@gsb.columbia.edu}}
} 

\ABSTRACT{%
We consider a sequential subset selection problem under parameter uncertainty, where at each time step, the decision maker selects a subset of cardinality $K$ from $N$ possible  items (arms), and observes a (bandit) feedback in the form of the index of one of the items in said subset, or none. Each item in the index set is ascribed a certain value (reward), and the feedback is governed by a Multinomial Logit (MNL) choice model whose parameters are a priori unknown.  The objective of the decision maker is to maximize the expected cumulative rewards over a finite horizon $T$, or alternatively, minimize the regret relative to an oracle that knows the MNL parameters.  We refer to this as the MNL-Bandit problem. This problem is representative of a larger family of exploration-exploitation problems that involve a combinatorial objective, and arise in several important application domains. We present an approach to adapt Thompson Sampling to this problem and show that it achieves near-optimal regret as well as attractive numerical performance.  
}%


\MSCCLASS{}
\ORMSCLASS{Primary: ; secondary: }
\HISTORY{}

\maketitle

\section{Introduction.}
 In the traditional stochastic multi-armed Bandit (MAB) problem, the decision maker selects one of, say,  $N$ arms in each round and receives feedback in the form of a noisy reward characteristic of that arm. Regret minimizing strategies are typically based on the principle of optimism in the face of uncertainty, a prime example of which are the family of upper confidence bound policies (UCB), which allow the player to learn the identity of the best arm through sequential experimentation, while concurrently not spending ``too much'' of the sampling efforts on the sub-optimal arms. In this paper we consider a combinatorial variant of this problem where in each time step the player selects a bundle of $K$ arms, after which s/he gets to see the 
reward associated with one of the arms in that bundle, or observing no reward at all.  One can think of the ``no reward''  as the result of augmenting each bundle with a further index that belongs to a ``null arm'' that cannot be directly chosen but can be manifest as a feedback; this structure will be further motivated shortly.    The identity of the arm within the bundle that yields the reward observation (or the ``null'' arm that yields no observation) is determined by means of a probability distribution on the index set of cardinality $K+1$ (the $K$ arms plus the ``null'' arm). In this paper the distribution is specified by means of a multinomial logit model (MNL); hence the name MNL-Bandit.   

A possible interpretation of this MNL-Bandit problem is as follows. A decision maker is faced with the problem of determining which  subset (of at most cardinality $K$) of $N$ items  to present  to users that arrive sequentially, where user preferences for said items are unknown. Each user either selects one of the items s/he is offered or selects none (the ``null arm'' option described above).  Every item presents some reward which is item-specific.  Based on the observations of items users have selected,  the decision maker needs to ascertain the composition of the ``best bundle,'' which involves balancing an exploration over bundles to learn the users' preferences, while simultaneously exploiting the bundles that exhibit good reward. (The exact mathematical formulation is given below.)  A significant challenge here is the combinatorial nature of the problem just described, as the space of possible subsets of cardinality $K$ is exponentially large, and for reasonable sized time horizons cannot be efficiently explored.  

The problem as stated above is not new, but there is surprisingly little antecedent literature on it; the review below will expound on its history and related strands of work.  It arises in many real-world instances, perhaps most notably in display-based online advertising.  Here the publisher has to select a set of advertisements to display to users.  Due to competing ads, the click rates for an individual ad depends on the overall subset of ads to be  displayed; this is referred to as a {\it substitution} effect.  For example, consider a user presented with two similar vacation packages from two different sources. The user's  likelihood of clicking on one of the ads in this scenario, would most likely differ from the situation where one of the ads is presented as a standalone. Because every advertisement is valued differently from the publisher's perspective,  the set of ads selected for display has a significant impact on revenues. A similar problem arises in online retail settings, where the retailer need to select a subset (assortment) of products to offer. Here demand for a specific product is influenced by the assortment of products offered. To capture these substitution effects,  choice models are often used to specify user preferences in the form of a probability distribution over items in a subset.  

The MNL-Bandit  is a natural way to cast the exploration-exploitation problem discussed above into a well studied machine learning paradigm, and allows to more easily adapt algorithmic ideas developed in that setting.  In particular, this paper focuses on a Thompson Sampling (TS) approach to the MNL-Bandit problem. This is  primarily motivated by the attractive empirical properties that have been observed over a stream of recent papers in the context of TS versus more traditional approaches such as upper confidence bound policies (UCB). For the MNL-Bandit this has further importance given the combinatorial nature of the dynamic optimization problem one is attempting to solve.  One of the main contributions of the present paper is in highlighting the salient features of TS that need to be adapted or customized to facilitate the design of an algorithm in the MNL-Bandit, and to elucidate  their role in proving regret-optimality for this variant of TS. To the best of our knowledge some of these ideas are new in the TS-context, and can hopefully extend its scope to combinatorial-type problems that will go beyond the MNL-Bandit. 




\section { Problem Formulation.} To formally state  our problem,  consider an option space containing $N$ distinct elements, indexed by $1,2,\ldots,N$ and their values denoted by  $r_1,\ldots,r_N$, with $r$ mnemonic for reward, though we will also use the term {\it revenue} in this context.   Since the user need not necessarily choose any of the options presented, we model this ``outside alternative'' as an additional item denoted with an index of ``0'' which augments the index set. We assume that for any offer set,  $S \subset \{1,\ldots,N\}$, the user will be selecting only one of the offered alternatives or item $0$, and this selection is given by a Multinomial Logit  (MNL) choice model.  Under this model, the probability that a user chooses item $i \in S$ is given by, \vspace{-0.05in}
\begin{eqnarray}\label{choice_probabilities} 
p_i(S) =\begin{cases}
 \displaystyle \frac{v_i}{v_0+\sum_{j \in S}v_j}, &\quad \text{if} \; i \in S \cup \{0\}\\
0, & \quad \; \text{otherwise},\vspace{-0.05in}
\end{cases}
\end{eqnarray}
where $v_i$ is a parameter of the MNL model corresponding to item $i$. Without loss of generality, we can assume that $v_0 = 1$.  (The focus on MNL is due to its prevalent use in the context of modeling substitution effects, and its tractability; see further discussion in related work.)
  
Given the above,  the expected revenue corresponding to the offer set $S$, $R(S)$ is given by
\begin{eqnarray}\vspace{-0.05in}
R(S, \mb{v}) = \sum_{i\in S} r_i p_i(S) = \sum_{i \in S} \frac{r_iv_i}{1+\sum_{j \in S} v_j}.\label{MNL_revenue}\vspace{-0.05in}
\end{eqnarray}
and the corresponding {\it static} optimization problem, i.e.,  when the parameter vector $\mb{v} =  (v_0,\ldots,v_N)$ and henceforth, $p_i(S)$ is known a priori,  is given by,   \vspace{-0.05in}
\begin{eqnarray}
\textstyle {\max}\left\{\;{R}(S,\mb{v}) \Big | |S| \leq K\right\}.\label{assort_opt}\vspace{-0.05in}
\end{eqnarray}
The cardinality constraints specified above, arise naturally in many applications. Specifically, a publisher/retailer is constrained by the space for advertisements/products and has to limit the number of ads/products that can be displayed. 

Consider a time horizon $T$, where a subset of items can be offered at time periods $t = 1,\ldots,T$.  Let $S^*$ be the {\it offline optimal offer set} for \eqref{assort_opt} under full information, namely,  when the values of $p_i(S)$, as given by \eqref{choice_probabilities}, are known a priori. In the MNL-Bandit,  the decision maker does not know the values of $p_i(S)$ and can only make sequential offer set decisions, $S_1,\ldots,S_T$,  at times $1,\ldots,T$, respectively. The objective is to design an algorithm that selects  a (non-anticipating)  sequence of offer sets in a path-dependent manner (namely, based on past choices and observed responses) to maximize cumulative expected revenues over the said horizon, or alternatively, minimize the {\it regret} defined  as \vspace{-0.05in}
\begin{eqnarray}\label{Regret}
\textstyle \Reg(T,\mb{v}) = \Ex\left[\sum_{t=1}^T R(S^*,\mb{v}) - R(S_t,\mb{v})\right],\vspace{-0.05in}
\end{eqnarray}
where $R(S,\mb{v})$ is the expected revenue when the offer set is $S$, and is as defined in \eqref{MNL_revenue}.
Here we make explicit the dependence of regret on the time horizon $T$ and the parameter vector $\mb{v}$ of the MNL model that determines the user preferences and choices.

\vspace{2mm}
\noindent {\bf Outline}. We review related literature and describe our contributions in Section~\ref{sec:relatedWork}. In Section \ref{sec:ourAlgorithm}, we present our adaptations of the Thompson Sampling algorithm for the \banditMNL, and in Section \ref{sec:regretAnalysis}, we prove our main result that our algorithm achieves an $\tilde{O}(\sqrt{NT}\log{TK})$ regret upper bound. {Section~\ref{sec:computations} demonstrates the empirical efficiency of our algorithm design.}

\section{Related Work and Overview of Contribution.}
\label{sec:relatedWork}
A basic pillar in the MNL-Bandit problem is the MNL choice model, originally introduced (independently) by  \citet{Luce59} and  \citet{Plackett75};  see also \citet{train2003discrete, McFadden78,BenLerman85} for further discussion and survey of other commonly used choice models. This model is by far the most widely used choice model insofar as capturing substitution effects that are a significant element in our problem. 
Initial motivation for this traces to online retail, where a retailer has to decide on a subset of items to offer from a universe of substitutable products for display. In this context, \citet{rusme} and \citet{saure} were the first two papers we are aware of, to consider a dynamic learning problem, in particular, focusing on minimizing regret under the MNL choice model. Both papers  develop an ``explore first and exploit later'' approach. 
Assuming knowledge of the ``gap" between the optimal and the next-best assortment,  they show an asymptotic $O(N\log{T})$ regret bound. (This assumption is akin to the ``separated arm" case in the MAB setting.)  It is worth noting that the algorithms developed in those papers require a priori  knowledge of this gap as a tuning input, which makes the algorithms parameter dependent.  
In a more recent paper,  \citet{agrawalnear} show how to exploit specific characteristics of the MNL model to develop a policy based on the principle of ``optimism under uncertainty'' (UCB-like algorithm, see \citet{auer}) which does not rely on the a priori knowledge of this gap or separation information and achieves a worst-case regret bound of $O(\sqrt{NT\log T})$. A regret lower bound of $\Omega(\sqrt{NT/K})$ for this problem is also presented in this work, which was subsequently improved to $\Omega(\sqrt{NT})$ in a recent work by \citet{xichen}.

It is widely recognized that UCB-type algorithms that optimize the worst case regret typically tend to spend ``too much time'' in the exploration phase, resulting in poor performance in practice (regret-optimality bounds notwithstanding).  To that end, several studies (\citet{chapelle}, \citet{graepel2010web}, \citet{may2012optimistic}) have demonstrated that TS significantly outperforms the state of the art methods in practice. Despite being easy to implement and often empirically superior, TS based algorithms are hard to analyze and theoretical work on TS is limited. To the best of our knowledge, \citet{agrawal2013further} is the first work to provide a finite time  worst-case regret bounds for the MAB problem that are independent of problem parameters. 

A naive translation of the \banditMNL~problem to an  MAB-type  setting would create $N \choose K$ ``arms'' (one for each offer set of size $K$).  For an ``arm'' corresponding to subset $S$, the reward is given by $R(S)$  (see \ref{assort_opt}). Managing this exponentially large arm space is prohibitive for obvious reasons. Popular extensions of MAB for ``large scale''  problems include the linear bandit (e.g., \citet{auer2003}, \citet{linucb}) for which  \citet{agrawal2013thompson} present a TS-based algorithm and provide finite time regret bounds. However, these approaches do not apply directly to our problem, since the revenue corresponding to each offered set is not linear in problem parameters. Moreover, for the regret bounds in those settings to be attractive, the dimension $d$ of parameters should be small, this dimension would be $N$ here. \citet{gopalan2014thompson} consider a variant of MAB where one can play a subset of arms in each round and the expected reward is a function of rewards of the arms played. This setting is similar to the \banditMNL, though the regret bounds they develop are dependent on the instance parameters as well as the number of possible actions which can be large in our combinatorial problem setting.  \citet{russo2014learning} consider a TS based learning algorithm for parametric bandit problems and provide finite time Bayesian regret bounds. Though \banditMNL\;can be formulated as a parametric bandit problem, Bayesian regret is a weaker notion of regret than the worst case regret, which is the focus of our work. Moreover, the computational tractability of updating the posterior in both the approaches (\citet{gopalan2014thompson} and \citet{russo2014learning}) is not immediately clear. \citet{russo2018tutorial} presents efficient heuristics to approximate the TS algorithm considered in \citet{russo2014learning}. However, it is not immediately clear if these approximate TS based approaches facilitate theoretical anaylsis. 
\paragraph{Our Contributions.}  In this work, relying on structural properties of the MNL model,  we develop a TS approach that is computationally efficient and yet achieves parameter independent (optimal in order) regret bounds. 
 Specifically, we present a computationally efficient TS algorithm for the \banditMNL~which uses a prior distribution on the parameters of the MNL model such that the posterior update under the MNL-bandit feedback is tractable. 
A key ingredient in our approach is a two moment approximation of the posterior  and the ability to judicially correlate samples, which is done by embedding the two-moment approximation in a normal family.   We show that our algorithm achieves a worst-case (prior-free) regret bound of $O(\sqrt{NT}\log{TK})$ under a mild assumption that $v_0\ge v_i$ for all $i$ (more on the practicality of this assumption later in the text); the bound is non-asymptotic, the ``big oh" notation is used for brevity.   This regret bound is independent of the parameters of the MNL choice model and hence holds uniformly over all problem instances.  The regret is comparable to the existing upper bound of $O(\sqrt{NT})$ provided by \citet{agrawalnear}, yet the numerical results demonstrate that our Thompson Sampling based approach significantly outperforms the UCB-based approach of \citet{agrawalnear}. Furthermore, the regret bound is also comparable to the lower bound of $\Omega(\sqrt{NT})$ established by \citet{xichen} under the same assumption, suggesting the optimality of our algorithm.  The methods developed in this paper highlight some of the key challenges involved in adapting the TS approach to the \banditMNL, and present a blueprint to address these issues that we hope will be more broadly applicable,  and form the basis for further work in the intersection of combinatorial optimization and machine learning.

\section{Algorithm.}
\label{sec:ourAlgorithm}
In this section, we describe our posterior sampling (aka Thompson Sampling) based algorithm for the \banditMNL~problem. The basic structure of Thompson Sampling involves maintaining a posterior on the unknown problem parameters, which is updated every time new feedback is obtained. In the beginning of every round, a sample set of parameters is generated from the current posterior distribution, and the algorithm chooses the best option according to these sample parameters. Due to its combinatorial nature, designing an algorithm in this framework for the \banditMNL~problem involves several new challenges as we describe below, along with our algorithm design choices to address them.
%

\subsection{Challenges and key ideas.}


\subsection*{ Conjugate priors for the MNL parameters.}
In the \banditMNL~problem, there is one unknown parameter $v_i$  associated with each item. 
To adapt the TS algorithm for our problem, we would need to maintain a joint posterior for $(v_1,\ldots,v_N)$. However, updating such a joint posterior is non-trivial since the feedback observed in every round is 
a sample from multinomial choice probability, ${v_i}/({1+\sum_{j\in S} v_j})$, which clearly depends on the subset  $S$ offered in that round. In particular, even if we initialize with an independent prior from a popular analytical family such as multivariate Gaussian, the posterior distribution after observing the MNL choice feedback can have a complex description. To address this, we leverage a sampling technique introduced in \citet{agrawalnear} that allows us to decouple individual parameters from the MNL choice feedback by obtaining unbiased estimates of these parameters. We utilize these unbiased estimates to efficiently maintain independent conjugate priors for the parameters $v_i$ for each $i$. Details of the resulting TS algorithm are presented in Algorithm \ref{learn_algo} in Section \ref{sec:vanilla}.

\medskip \noindent {\bf Posterior approximation and Correlated sampling.} Algorithm \ref{learn_algo} presents unique challenges in theoretical analysis. A worst case regret analysis of Thompson Sampling based algorithms for MAB typically proceeds by showing that the best arm is optimistic at least once every few steps, in the sense that its sampled parameter is better than the true parameter. Such a proof approach for our combinatorial problem requires that every few steps, all the $K$ items in the optimal offer set have sampled parameters that are better than their true counterparts. However, Algorithm \ref{learn_algo} samples the posterior distribution for each parameter {\it independently} in each round.  This makes the probability of being optimistic exponentially small in $K$. 

We address this challenge by employing {\it correlated sampling} across items. To implement correlated sampling, we find it useful to approximate the Beta posterior by a Gaussian distribution with approximately the same mean and variance as the Beta distribution; what was referred to in the introduction as a two-moment approximation.  This allows us to generate correlated samples from the $N$ Gaussian distributions as linear transforms of a single standard Gaussian. Under such correlated sampling, the probability of all $K$ optimal items to be simultaneously optimistic is a constant, as opposed to being exponentially small (in $K$) in the case of independent samples. However, such correlated sampling reduces the overall variance of the maximum of $N$ samples severely, thus reducing exploration. We boost the variance by taking $K$ samples instead of a single sample of the standard Gaussian. The resulting variant of Thompson Sampling algorithm is presented in Algorithm \ref{learn_algo_normal} in Section \ref{sec:variant}. We prove near-optimal regret bound for this algorithm in Section \ref{sec:regretAnalysis}.

\subsection{A TS algorithm with independent conjugate Beta priors} 
\label{sec:vanilla}
Here, we present a first version of our Thompson sampling algorithm, which serves as an important building block for our main algorithm. 
In this  version of the algorithm, we maintain a Beta posterior distribution for each item $i=1,\ldots, N$, which is updated as we observe users' choice of items from the offered subsets.  A key challenge here is to design priors that can be efficiently updated on observing  user choice feedback, in order to obtain increasingly accurate estimates of parameters $\{v_i\}$.  To address this, we use the sampling technique introduced in \citet{agrawalnear} to decouple individual parameters from the complex MNL feedback. The idea is to offer a set $S$ multiple times; in particular, a chosen $S$ is offered repeatedly until an ``outside option'' is picked (in the motivating application discussed earlier, this corresponds displaying the same subset of ads until we observe a user who does not  click on any of the displayed ads). Proceeding in this manner, the average number of times an item $i$ is selected provides an unbiased estimate of parameter $v_i$. Moreover, the number of times an item $i$ is selected  is also independent of the displayed set and is a geometric distribution with success probability $1/(1+v_i)$ and mean $v_i$. 
This observation is used as the basis for our epoch based algorithmic structure and our choice of prior/posterior, as a conjugate to this geometric distribution. The following lemmas provide important building blocks for our construction. Their proofs have been deferred to the appendix. 
\begin{restatable}[Agrawal et al. 2016]{lemma1}{agrawalUnbiased}
\label{unbiased_estimate}
~Let $\tilde{v}_{i,\ell}$ be the number of times an item $i\in S_\ell$ is picked when the set $S_\ell$ is offered repeatedly until no-click (outside option is picked). Then, $\tilde{v}_{i,\ell}, \forall \ell,i$ are  i.i.d geometrical  random variables with success probability $\frac{1}{1+v_i}$, and expected value $v_i$. \vspace{-0.1in}
\end{restatable}

\begin{restatable}[Conjugate Priors]{lemma1}{conjugateprior}
\label{conjugate_prior}
~For any $\alpha >3, \beta > 0$, let $X_{\alpha,\beta} = \frac{1}{{\sf Beta}(\alpha,\beta)}-1$ and $f_{\alpha,\beta}$ be a probability distribution of the random variable $X_{\alpha,\beta}$. If  $v_i$ is distributed as $f_{\alpha,\beta}$ and $\tilde{v}_{i,\ell}$ is a geometric random variable with success probability $\frac{1}{v_i+1}$, then we have,\vspace{-0.1in}
$$\textstyle \bb{P}\left(v_i \Big| \tilde{v}_{i,\ell} = m\right) = f_{\alpha+1,\beta+m}(v_i).$$ 
\end{restatable}

\medskip \noindent {\bf Epoch based offerings:} Our algorithm proceeds in epochs $\ell=1,2, \ldots$. An epoch is a group of consecutive time steps, where a set $S_\ell$ is offered repeatedly until the outside option is picked in response to offering $S_\ell$. The set $S_\ell$ to be offered in an epoch $\ell$ is picked in the beginning of the epoch based on the sampled parameters from the current posterior distribution; the construction of these posteriors and choice of $S_\ell$ is described in the next paragraph. We denote the group of time steps in an epoch as $\ep{E}_\ell$, which includes the time step at which an outside option was preferred.

\medskip \noindent {\bf Construction of conjugate prior/posterior:} From Lemma \ref{unbiased_estimate}, we have that for any epoch $\ell$ and for any item $i \in S_\ell$, the estimate $\tilde{v}_{i,\ell}$, the number of picks of item $i$ in epoch $\ell$ is geometrically distributed with success probability $1/(1+v_i).$
Suppose that the prior distribution for parameter $v_i$  in the beginning of an epoch $\ell$ is same as that of 
$$X_i = \frac{1}{{\sf Beta}(n_i,V_i)} - 1,$$
where ${\sf Beta}(n_i,V_i)$ is the Beta random variable with parameters $n_i$ and $V_i$. In Lemma \ref{conjugate_prior}, we show that after observing the geometric variable $\tilde{v}_{i,\ell} =m,$ the posterior distribution of $v_i$ is same as that of,
$$X_i' = \frac{1}{{\sf Beta}(n_i+1,V_i+m)} - 1.$$
Therefore, we use the distribution of  $\frac{1}{{\sf Beta}(1,1)} - 1$ as the starting prior for $v_i$, and then, in the beginning of epoch $\ell$, the posterior is distributed as $ \frac{1}{{\sf Beta}(n_i(\ell),V_i(\ell))} - 1$, with $n_i(\ell)$ being the number of epochs the item $i$ has been offered before epoch $\ell$ (as part of an assortment), and $V_i(\ell)$ being the number of times it was picked by the user. 

\medskip \noindent {\bf Selection of subset to be offered:}
To choose the subset to be offered in epoch $\ell$, 
the algorithm samples a set of parameters ${\mu}_{1}(\ell), \ldots, {\mu}_{N}(\ell)$ independently from the current posteriors and finds the set that maximizes the expected revenue as per the sampled parameters. In particular, the set $S_{\ell}$ to be offered in epoch $\ell$ is chosen as:\vspace{-0.05in}
\begin{eqnarray}\label{eq:find_assort}
 \textstyle S_\ell := \underset{|S| \leq K}{\text{arg~max}} R(S, \pmb{\mu}(\ell))\vspace{-0.05in}
\end{eqnarray}
There are efficient polynomial time algorithms available to solve this optimization problem (e.g., refer to \citet{davis2013assortment}, \citet{avadhanula2016tightness} and \citet{rusme}).

The details of our procedure are provided in Algorithm \ref{learn_algo}. 

\begin{algorithm}[h]
\caption{$\text{A TS algorithm for \banditMNL~with Independent Beta priors}(K, \alpha, \beta)$}\label{learn_algo}
\begin{algorithmic}
\vspace{2pt}
\State \textbf{Initialization:} For each item $i=1,\cdots,N$, ${V}_i = 1$, $n_i = 1$. 
 \vspace{2pt}
\State $t = 1$, keeps track of the time steps 	
\vspace{2pt}
\State $\ell = 1$, keeps count of total number of epochs
\vspace{2pt} 	
\While{$t \le T$}
 \begin{enumerate}[leftmargin=.4in]
 	\vspace{2pt}
	\item[(a)] {(\it Posterior Sampling)} For each item $i = 1,\cdots, N$,  sample $\theta_i(\ell)$ from the ${\sf Beta}(n_i, V_i)$ and compute ${\mu}_i(\ell) = \frac{1}{\theta_i(\ell)} - 1$
	\vspace{2pt}
	\item[(b)] {(\it Subset Selection)} Compute $S_{\ell} = \underset{|S| \leq K}{\text{arg~max}}\;R(S, \pmb{\mu}(\ell)) = \frac{\sum_{i \in S}r_i {\mu}_i(\ell)}{1+\sum_{j \in S} {\mu}_{j}(\ell)}$
	\vspace{2pt}
	\item[(c)] {(\it Epoch-based offering)}
	\item[] \textbf{repeat}
		\begin{itemize}[leftmargin=.3in]
			\vspace{2pt}
			\item[] Offer the set $S_{\ell}$, and observe the user choice ${c_t}$; 
			\vspace{2pt}
			\item[] Update $\ep{E}_\ell = \ep{E}_\ell \cup t$, time indices corresponding to epoch $\ell$; $t = t+1$
		\end{itemize}
	\item[] \textbf{until} {${c}_t = 0$}
	\vspace{2pt}
	\item[(d)] {(\it Posterior update)}
		\begin{itemize}[leftmargin=0.3in]
			\vspace{2pt}
			\item[] For each item $i \in S_\ell$, compute $\tilde{v}_{i,\ell} = \sum_{t \in \ep{E}_\ell} \mathbb{I}({c}_t = i)$, no. of picks of item $i$ in epoch $\ell$.  
			\vspace{2pt}
			\item[]	Update ${V}_i$ = $V_i + \tilde{v}_{i,\ell}$, $n_i = n_i+1$, $\ell = \ell + 1.$
		\end{itemize}
\end{enumerate}
\EndWhile
\end{algorithmic}
\end{algorithm}

\subsection{ A TS algorithm with posterior approximation and correlated sampling}
\label{sec:variant}
Motivated by the challenges in theoretical analysis of Algorithm~\ref{learn_algo} described earlier, in this section we design a variant, Algorithm~\ref{learn_algo_normal}. The main changes in this version of the algorithm are the posterior approximation by means of a Gaussian distribution, correlated sampling, and taking multiple samples (``variance boosting''). We describe each of these changes below. First, we present the following result that helps us in approximating the posterior. Proof of the result has been deferred to the appendix. 

\begin{restatable}[Moments of the Posterior Distribution]{lemma1}{posteriormoment}\label{posterior_moment}
~If $X$ is a random variable distributed as ${\sf Beta}(\alpha,\beta)$, then 
\begin{equation*}
\begin{aligned}
\textstyle \bb{E}\left(\frac{1}{X}-1\right) &= \textstyle \frac{\beta}{\alpha-1},\;\;\;\text{and}\;\;
\textstyle {\sf Var}\left(\frac{1}{X}-1\right) & = \textstyle \frac{\frac{\beta}{\alpha-1}\left(\frac{\beta}{\alpha-1} + 1\right)}{\alpha-2}.\vspace{-0.05in}
\end{aligned}
\end{equation*}
\end{restatable}

\medskip \noindent {\bf Posterior approximation:} We approximate the posterior distributions  used in Algorithm \ref{learn_algo} for the MNL parameters $v_i$, by Gaussian distributions with approximately the same mean and variance (refer to Lemma \ref{posterior_moment}). 
In particular, let
\begin{eqnarray}\label{eq:first_var}
\hat{v}_i(\ell) := \frac{V_i(\ell)}{n_i(\ell)}, \ \ \hat{\sigma}_i(\ell) := \sqrt{\frac{50\hat{v}_i(\ell)(\hat{v}_i(\ell)+1)}{n_i(\ell)}} + 75\frac{\sqrt{\log{TK}}}{n_i(\ell)},
\end{eqnarray}
where $n_i(\ell)$ is the number of epochs the item $i$ has been offered before epoch $\ell$ (as part of an assortment), and $V_i(\ell)$ being the number of times it was picked by the user. We will use $\ep{N} \left(\hat{v}_i(\ell), \hat \sigma^2_i(\ell)\right)$ as the posterior distribution for item $i$ in the beginning of epoch $\ell.$ The Gaussian approximation of the posterior is employed to facilitate efficient correlation of posterior samples. The correlated sampling plays a key role in avoiding the theoretical challenges associated with independent posteriors in \mbox{Algorithm \ref{learn_algo}.}

\medskip \noindent {\bf Correlated sampling:} Given the posterior approximation by Gaussian distributions, we correlate the samples by using a common standard normal variable and constructing our posterior samples as an appropriate transform of this common standard normal. More specifically, in the beginning of an epoch $\ell$, we generate a sample from the standard normal distribution, $\theta \sim \ep{N}\left(0,1\right)$ and the posterior sample  for item $i$,  is generated  as $\hat{v}_i(\ell) + \theta \hat{\sigma}_i(\ell)$. This allows us to generate sample parameters for $i=1,\ldots, N$ that are either simultaneously high or simultaneously low, thereby, boosting the probability that the sample parameters for {\it all} the $K$ items in the best assortment are optimistic (the sampled parameter values are higher than the true parameter values).  

 \medskip \noindent {\bf Multiple ($K$) samples:} The correlated sampling decreases the joint variance of the sample set. More specifically, for any epoch $\ell$, we have that 
 $$\textstyle {\sf Var} \left(\underset{i = 1,\cdots,N}{\max} \left\{\hat{v}_i(\ell) + \theta \hat{\sigma}_i(\ell) \right\}\right) < \textstyle {\sf Var} \left(\underset{i = 1,\cdots,N}{\max} \left\{\hat{v}_i(\ell) + \theta_i \hat{\sigma}_i(\ell) \right\}\right),$$  where $\theta_i$ are sampled independently from the standard normal distribution for every $i$. In order to boost this joint variance and ensure sufficient exploration, we generate multiple sets of samples. In particular, in the beginning of an epoch $\ell$, we generate $K$ independent samples from the standard normal distribution, 
$\theta^{(j)} \sim \ep{N}(0,1), j=1,\ldots, K$. And then, the $j^{th}$ sample set is generated as: \vspace{-0.05in}
$$\mu^{(j)}_i(\ell) = \hat{v}_i(\ell) + \theta^{(j)} \hat{\sigma}_i(\ell), \ \ \ \ i=1,\ldots, N,\vspace{-0.05in}$$ 
and we use the highest valued samples\vspace{-0.07in}
$$\mu_i(\ell) = \underset{j = 1,\cdots,K}{\max} \ \ \mu^{(j)}_i(\ell), \forall i,$$ to decide the assortment to offer in epoch $\ell$, 
$$ S_{\ell} = \arg \max_{S \in {\cal S}} R(S, \pmb{\mu}(\ell)) \vspace{-0.07in}$$

We summarize the steps in Algorithm \ref{learn_algo_normal}. 
Here, we also have an ``initial exploration period," where for every item $i$, we offer a set containing only $i$ until the user selects the outside option. 

\begin{algorithm}[h]
\caption{$\text{A TS algorithm for \banditMNL~with Gaussian approximation and correlated sampling}$}\label{learn_algo_normal}
\begin{algorithmic}
\State \textbf{Initialization:} $t = 0$, $\ell = 0$, $n_i = 0$ for all $i = 1, \cdots, N$. 
\vspace{2pt}
\For{ each item,  $i=1,\cdots,N$}
\begin{itemize}
\vspace{2pt}
\item[]  Display item $i$ to users until the user selects the ``outside option''. Let $\tilde{v}_{i,1}$ be the number of times item $i$ was offered. \; Update: $V_i = \tilde{v}_{i,1} - 1$, $t = t + \tilde{v}_{i,1}$, $\ell = \ell + 1$ and $n_i = n_i + 1$. 
\end{itemize}
\vspace{2pt}
\EndFor
\vspace{2pt} 	
\While{$t \le T$}
 \begin{enumerate}[leftmargin=.4in]
 	\vspace{2pt}
	\item[(a)] {(\it Correlated Sampling)} \textbf{for} $j = 1,\cdots,K$
	\begin{itemize}[noitemsep]
	\vspace{2pt}
	\item[] Sample $\theta^{(j)}(\ell) $ from the distribution $\ep{N}\left(0,1\right)$; ~ update $\hat{v}_i = \frac{V_i}{n_i}$.
	\vspace{2pt}
	\item[] For each item $i\leq N$, compute $\mu^{(j)}_i(\ell) = \hat{v}_i + \theta^{(j)}(\ell)\cdot \left(\sqrt{\frac{50\hat{v}_i(\hat{v}_i+1)}{n_i}}  + \frac{75\sqrt{\log{TK}}}{n_i}\right)$.
	\vspace{3pt}
	\item[] \textbf{end}
	\end{itemize}
	\vspace{5pt}
	\item[] For each item $i \leq N$, compute $\mu_i(\ell) = \underset{j = 1, \cdots, K}{\max } \mu_i^{(j)}(\ell)$
	\vspace{3pt}
	\item[(b)] {\it (Subset selection)} Same as step (b) of Algorithm \ref{learn_algo}.
	\vspace{3pt}
	\item[(c)]{\it (Epoch-based offering)} Same as step (c) of Algorithm \ref{learn_algo}.
	\vspace{3pt}
	\item[(d)]{\it (Posterior update)} Same as step (d) of Algorithm \ref{learn_algo}. 
	\vspace{3pt}
\end{enumerate}
\EndWhile
\end{algorithmic}
\end{algorithm}

Intuitively, while the second moment approximation by Gaussian distribution and multiple samples in Algorithm \ref{learn_algo_normal} may make posterior converge slower and increase exploration, the correlated sampling may compensate for these effects by reducing the variance of the maximum of $N$ samples, and therefore reducing the overall exploration effort. {In Section \ref{sec:computations}, we illustrate some of these insights through some preliminary numerical simulations, where correlated sampling performs significantly better compared to independent sampling, and posterior approximation by Gaussian distribution has little effect.}

\section{Regret Analysis} \label{sec:regretAnalysis}
We prove an upper bound on the regret of Algorithm \ref{learn_algo_normal} for the  \banditMNL~problem, under the following assumption.
\begin{assumption1}\label{assumption1}
{\normalfont ~ For every item $i \in \{1,\ldots, N\}$, the MNL parameter $v_i$ satisfies $v_i \leq v_0 = 1$.}
\end{assumption1}
This assumption is equivalent to the outside option being more preferable to any other item. This assumption holds for many applications like display advertising, where users do not click on any of the displayed ads more often than not. Our main theoretical result is the following upper bound on the regret of Algorithm \ref{learn_algo_normal}.
\begin{theor1}\label{main_result}
~For any instance $\mb{v} = (v_0, \cdots, v_N)$ of the \banditMNL~problem with $N$ products,  $r_i\in [0,1]$, and satisfying Assumption \ref{assumption1}, the regret of Algorithm~\ref{learn_algo_normal} in time $T$ is bounded as,
\begin{equation*} 
\reg(T,\mb{v}) \leq C_1\sqrt{NT}\log{TK} + C_2N\log^2{TK}, 
\end{equation*}
where $C_1$ and $C_2$ are absolute constants (independent of problem parameters).
\end{theor1}
\subsection{Proof Sketch}\label{proof_outline_thm1}
We break down the expression for total regret 
$$\reg(T,\mb{v}):=\Ex\left[\sum_{t=1}^T R(S^*, \mb{v}) - R(S_t, \mb{v})\right],$$
into regret per epoch, and rewrite it as follows:
\begin{eqnarray}\label{eq:regret_breakdown_main}
\Reg(T,\mb{v})  = \underbrace{ \bb{E}\left[\sum_{\ell=1}^L |\ep{E}_{\ell}| \left(R(S^*, \mb{v}) - R(S_\ell, \pmb{\mu}(\ell)) \right)\right]}_{\Reg_1(T,\mb{v})} + \underbrace{ \bb{E}\left[\sum_{\ell=1}^L |\ep{E}_{\ell}| \left(R(S_\ell, \pmb{\mu}(\ell)) - R(S_\ell,\mb{v})\right)\right]}_{\Reg_2(T,\mb{v})},  
\end{eqnarray}
where $|\ep{E}_{\ell}|$ is the number of time steps in epoch $\ell$, and $S_\ell$ is the set repeatedly offered by our algorithm in epoch $\ell$. Then, we bound the two terms: $\Reg_1(T,\mb{v})$ and $\Reg_2(T,\mb{v})$ separately.

The first term $\Reg_1(T,\mb{v})$ is essentially the difference between the optimal revenue of the true instance and the optimal revenue of the sampled instance. 
Therefore, this term would contribute no regret if the revenues corresponding to the sampled instances were always optimistic, i.e. $R(S_\ell,\pmb{\mu}(\ell)) > R(S^*,\mb{v})$. Unlike optimism under uncertainty approaches like UCB, this property is not ensured by our Thompson Sampling based algorithm. To bound this term, we utilize anti-concentration properties of the posterior, as well as the dependence between samples for different items, in order to prove that at least one of our $K$ sampled instances is optimistic often enough. 

The second term $\Reg_2(T,\mb{v})$ captures the difference in the revenue of the offered set $S_\ell$ when evaluated on sampled parameters in comparison with the true parameters. This is bounded by utilizing the concentration properties of our posterior distributions. It involves showing that for the sets that are played often, the posterior will converge quickly, so that revenue on the sampled parameters will be close to that on the true parameters.

Before elaborating further on the proof details, we will first highlight three key results involved in proving Theorem \ref{main_result}. 


\medskip \noindent {\bf Structural properties of the optimal revenue.} The first step in our regret analysis is to leverage the structure of the MNL model to establish two key properties of the optimal expected revenue. These properties project the non-linear reward function of the MNL choice into its parameter space and help us focus on analyzing the posterior distribution of the parameters.
In the first property, which we refer to as restricted monotonicity, we note that the optimal expected revenue is monotone in the MNL parameters. In the second property, we present a Lipschitz property of the expected revenue function. In particular, we note that the difference between the expected revenue corresponding to two different MNL parameters is bounded in terms of the difference in individual parameters. Lemma \ref{structprop} provides the precise statement, we defer the proof to Appendix \ref{appsec:structprop}.
\begin{restatable}[Properties of the Optimal Revenue]{lemma1}{structprop}\label{structprop}
~Fix $\mb{v} \in \ep{R}_{+}^n$, let $S^* $ be an optimal assortment when the MNL are parameters are given by $\mathbf{v}$, i.e. $S^* = \underset{S : |S| \leq K}{\text{arg~max}}\;{R}(S,\mb{v}).$ For any $\mb{w} \in \ep{R}_{+}^n$, we have:
\begin{enumerate}
\item ({Restricted Monotonicity}) If $v_i \leq w_{i}$\;for all $i = 1, \cdots, N$. Then, $R(S^*, \mathbf{w}) \geq R(S^*, \mathbf{v}).$
\vspace{2pt}
\item ({Lipschitz}) $\left|R(S^*,\mb{v}) - R(S^*,\mb{w})\right| \leq \displaystyle \frac{ \sum_{i \in S^*}| v_i - w_i |}{1+ \sum_{j \in S^*}v_j}.$
\end{enumerate}
\end{restatable}

\medskip \noindent {\bf Concentration of the posterior distribution.} The next step in the regret analysis is to show that as items are offered an increasing number of times, the posterior distributions concentrate around their means, which in turn concentrate around the true parameters. More specifically, at the beginning of epoch $\ell$, we can show with high probability that for any item $i$, the difference between the sample from the posterior distribution, $\mu_i(\ell)$ (see Step (a) of Algorithm \ref{learn_algo_normal}) and the true parameter, $v_i$ is bounded by the sample variance, $\hat{\sigma}_i(\ell)$ (see \eqref{eq:first_var}), which decreases over time. Leveraging the Lipschitz property of the optimal revenue, this concentration of sample parameter around its true value will help us prove that the difference between the expected revenues of the offer set $S_\ell$ corresponding to the sampled parameters, $\pmb{\mu}(\ell)$ and the true parameters, $\mb{v}$ also becomes smaller with time. In particular, we have the following inequality with high probability
\begin{eqnarray}\label{eq:lipschitz_var_revenue}
\vspace{-1pt}
\left |R(S_\ell,\pmb{\mu}(\ell)) - R(S_\ell,\mb{v}) \right|\lesssim O\left(\sum_{i \in S_\ell} \hat{\sigma}_i(\ell)\right).
\end{eqnarray}
Lemma \ref{lem:lipschitz} in Appendix \ref{sec:lipschitz} provides the precise statement. 

\medskip \noindent {\bf Anti-Concentration of the posterior distribution.}  We will refer to an epoch $\ell$ as optimistic, if the revenue of the optimal set $S^*$ corresponding to the sampled parameters, $\pmb{\mu}(\ell)$ is at least as high as the revenue on true parameters, i.e. $R(S^*,\pmb{\mu}(\ell)) \geq R(S^*,\mb{v}).$ Since $S_\ell$ is an optimal set for the sampled parameters, we have $R(S_\ell,\pmb{\mu}(\ell)) \geq R(S^*,\mb{v})$. This suggests that as the number of optimistic epochs increases, the term $\Reg_1(T,\mb{v})$ decreases. 

The final and important technical component of our analysis is showing that there are only a ``small'' number of non-optimistic epochs. From the restricted monotonicity property of the optimal revenue (see Lemma \ref{structprop}), we have that an epoch $\ell$ is optimistic if every sampled parameter, $\mu_i(\ell)$ is at least as high as the true parameter $v_i$ for all the items $i$ in the optimal set $S^*.$  Noting that the posterior samples, $\mu^{(j)}_i(\ell)$, are generated from a Gaussian distribution, whose mean concentrates around the true parameter $v_i$, we can conclude that any sampled parameter will be greater than the true parameter with constant probability, i.e. $\mu^{(j)}_i(\ell) \geq v_i$. However, for an epoch to be optimistic, sampled parameters for \emph{all} the items in $S^*$ may need to be larger than the true parameters. This is where the correlated sampling feature of our algorithm plays a key role. Utilizing the dependence structure between samples for different items in the optimal set, and variance boosting provided by the sampling of $K$ independence sample sets, we prove an upper bound of roughly $1/K$ on the number of consecutive epochs between two optimistic epochs.  
Lemma \ref{epoch_length_analysis} provides the precise statement, we defer the proof to Appendix  \ref{app:analEpochLength}.
\begin{restatable}[Spacing of optimistic epochs]{lemma1}{optimism}
\label{epoch_length_analysis}
~Let 
$\ep{E}^{\sf An}(\tau)$ be the group of consecutive epochs between an optimistic epoch $\tau$ and the next optimistic epoch $\tau'$, excluding the epochs $\tau$ and $\tau'$. Then, for any $p \in [1,2]$, we have, \vspace{-0.05in}
$$\bb{E}^{1/p}\left[ \left|\ep{E}^{\sf An}(\tau) \right |^p  \right] \leq \frac{e^{12}}{K}+30^{1/p}.$$
\end{restatable}

We will now briefly discuss how the above properties are put together to bound $\Reg_1(T,\mb{v})$ and $\Reg_2(T,\mb{v})$. A complete proof is provided in Appendix \ref{app:mainProof}.

\medskip \noindent {\bf Bounding the first term $\Reg_1(T,\mb{v})$.} 

Firstly, by our assumption $v_0 \ge v_i, \forall i$, the outside option is picked at least as often as any particular item $i$. Therefore, it is not difficult to see that the expected value of epoch length $|\ep{E}_\ell|$ is bounded by $K+1$, so that $\Reg_1(T,\mb{v})$ is bounded as 
$$(K+1) \bb{E}\left(\sum_{\ell=1}^L R(S^*, \mb{v}) - R(S_\ell, \pmb{\mu}(\ell))\right).$$
Recall that for every optimistic epoch, we have that the set $S^*$ has at least as much revenue on the sampled parameters as on the true parameters. Hence, optimistic epochs don't contribute to this term.
To bound the contribution of the remaining epochs, we bound the individual contribution of any ``non-optimistic" epoch $\ell$ 
by relating it to the closest optimistic epoch $\tau$ before it. 
By definition of an optimistic epoch and by the choice of $S_\ell$ as the revenue maximizing set for the sampled parameters $\pmb{\mu}(\ell)$, we have 
$$R(S^*,\mb{v}) - R(S_\ell,\pmb{\mu}(\ell)) \le R(S_\tau, \pmb{\mu}(\tau)) - R(S_\ell,\pmb{\mu}(\ell)) \le R(S_\tau, \pmb{\mu}(\tau)) - R(S_\tau,\pmb{\mu}(\ell)).$$
We will utilize the concentration property of the posterior and the Lipschitz property of the revenue function to bound the difference in the revenue of the set $S_\tau$ corresponding to two different sample parameters: $\pmb{\mu}(\tau)$ and $\pmb{\mu}(\ell)$.  From  \eqref{eq:lipschitz_var_revenue}, the difference in the revenues can be bounded by the sum of sample variances $\hat \sigma_i(\tau) +\hat \sigma_i(\ell)$ and since the variance at the beginning of epoch $\tau$ is larger than the variance at the beginning of epoch $\ell$, we have,
$$\left|R(S_\tau, \pmb{\mu}(\tau)) - R(S_\tau,\pmb{\mu}(\ell))\right| \lesssim O\left(\sum_{i \in S_\tau}\hat \sigma_i(\tau)\right).$$
From the above bound, we have that the regret in non-optimistic epoch is bounded by the sample variance in the closest optimistic epoch before it. Utilizing the fact on an average there are only $1/K$ non-optimistic epochs (see Lemma \ref{epoch_length_analysis}) between any two consecutive optimistic epochs, we can bound the term $\Reg_1(T,\mb{v})$ as: 
$$\Reg_1(T,\mb{v}) \leq (K+1)O\left(\bb{E}\left[\sum_{\ell \in {\sf optimistic}} \frac{1}{K}\sum_{i\in S_\ell} \hat\sigma_i(\ell)\right]\right).$$
A bound of $\tilde{O}(\sqrt{NT})$ on the sum of these deviations can be derived, which will also be useful for bounding the second term, as discussed next.


\medskip \noindent {\bf Bounding the second term $\Reg_2(T,\mb{v})$.} 

Noting that the expected epoch length when set $S_\ell$ is offered is $ 1+\sum_{j\in S_\ell}v_i$, $\Reg_2(T,\mb{v})$ can be reformulated as
$$\Reg_2(T,\mb{v}) = \bb{E}\left[\sum_{\ell=1}^L (1+V(S_\ell)) \left(R(S_\ell,\pmb{\mu}(\ell)) - R(S_\ell,\mb{v})\right)\right],$$
 Again, as discussed above, using Lipschitz property of revenue function and the concentration properties of the posterior distribution, this can be bounded in terms of posterior standard deviation (see \eqref{eq:first_var})
$$ \Reg_2(T,\mb{v}) \lesssim O\left(\bb{E}\left[\sum_{\ell=1}^L\sum_{i \in S_\ell} \hat{\sigma}_i(\ell)\right]\right).$$

Overall, the above analysis on $\Reg_1$ and $\Reg_2$ implies roughly the following bound on regret
$$O(\sum_{\ell=1}^L\sum_{i \in S_\ell} \hat{\sigma}_i(\ell)) = {O}\left(\sum_{\ell=1}^L\sum_{i \in S_\ell}\sqrt{\frac{v_i}{n_i(\ell)}}  + \frac{1}{n_i(\ell)}\right)\log{TK} \leq {O}(\sum_{i=1}^N\log{TK}\sqrt{v_in_i}),$$
  where $n_i$ is total number of times $i$ was offered in time $T$. Then, utilizing the bound of $T$ on the expected number of total picks, i.e., $\sum_{i=1}^N v_in_i \leq T$, and doing a worst case scenario analysis, we obtain a bound of $\tilde{O}(\sqrt{NT})$ on $\Reg(T,\mb{v})$.

\section{Empirical study}\label{sec:computations}
In this section, we analyze the various design components of our Thompson Sampling approach through numerical simulations. The aim is to  isolate and understand the effect of individual features of our algorithm design like Beta posteriors vs. Gaussian approximation, independent sampling vs. correlated sampling, and single sample vs. multiple samples, on the practical performance. 


\begin{figure}[t]
\begin{center}
{\includegraphics[scale=0.25]{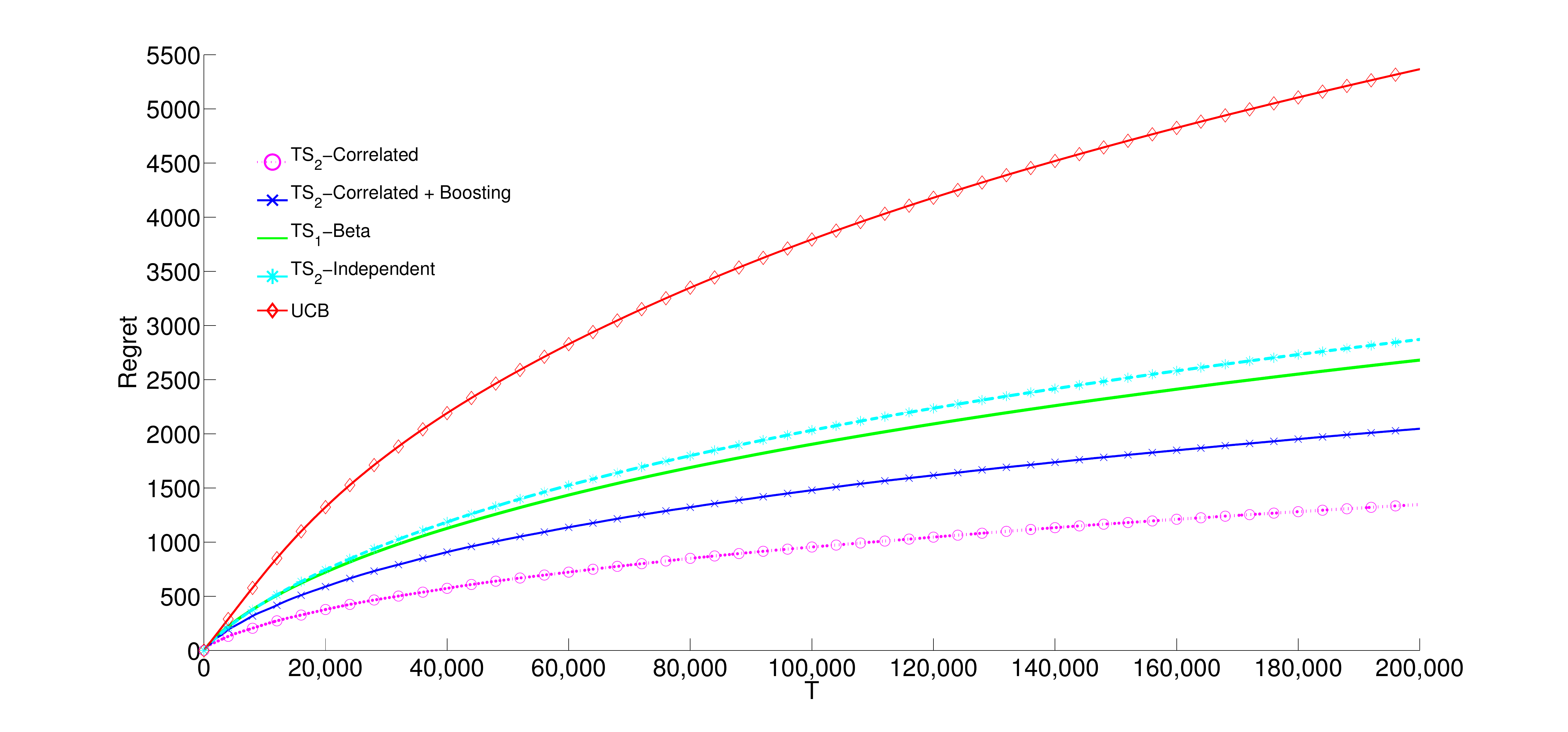}} \vspace{-0.2in}
\caption{Regret growth with $T$ for various heuristics on a randomly generated \banditMNL~instance with $N=1000, K=10$. 
\label{fig:regret_heuristic} \vspace{-0.2in}}
\end{center}
\end{figure}

We simulate an instance of \banditMNL~problem with $N=1000$, $K=10$ and $T=2\times10^5$, and the MNL parameters $\{v_i\}_{i = 1,\ldots,N}$ generated randomly from ${\sf Unif}[0,1]$. And, we compute the average regret based on $50$ independent simulations over the randomly generated instance.  In Figure \ref{fig:regret_heuristic}, we report performance of successive variants of TS: 
\begin{itemize}\setlength\itemsep{2pt}
\item[$i)$] basic version of TS with independent Beta priors, as described in Algorithm~\ref{learn_algo}, referred to as ${\sf TS_1\mbox{-} Beta}$, 
\item[$ii)$] Gaussian posterior approximation with independent sampling, referred to as ${\sf TS_2\mbox{-}Independent}$, 
\item[$iii)$] Gaussian posterior approximation with correlated sampling, referred to as ${\sf TS_2\mbox{-} Correlated}$, and finally, 
\item[$iv)$] Gaussian posterior approximation with correlated sampling and boosting by using multiple ($K$) samples, referred to as ${\sf TS_2\mbox{-} Correlated+Boosting}$, which is essentially the version with all the features of Algorithm~\ref{learn_algo_normal}. 
\end{itemize}

For comparison, we also present the performance of ${\sf UCB}$ approach in~\cite{agrawalnear}. We repeated this experiment on several randomly generated instances and a similar performance was observed. The performance of all the variants of TS is observed to be better than the UCB approach in our experiments, which is consistent with the other empirical evidence in the literature. 

Among the TS variants, the performance of ${\sf TS_1\mbox{-} Beta}$, i.e., the basic version with independent beta priors (essentially Algorithm~\ref{learn_algo}) is quite similar to ${\sf TS_2\mbox{-} Independent}$, the version with independent Gaussian (approximate) posteriors;
indicating that the effect of posterior approximation is minor. 
The performance of ${\sf TS_2\mbox{-} Correlated}$, where we generated correlated samples from the Gaussian distributions, is significantly better than all the other variants of the algorithm. This is consistent with our remark earlier that to adapt the Thompson sampling approach of the classical MAB problem to our setting, ideally we would like to maintain a joint prior over the parameters $\{v_i\}_{i=1,\ldots,N}$ and update it to a joint posterior on observing the bandit feedback. However, since this can be quite challenging and intractable, we used independent priors over the parameters. The superior performance of ${\sf TS_2\mbox{-} Correlated}$ demonstrates the potential benefits of considering a joint (correlated) prior/posterior in such settings with combinatorial arms. Finally, we observe that the performance of ${\sf TS_2\mbox{-}  Correlated+Boosting}$, where an additional ``variance boosting" is provided through $K$ independent samples, is worse than ${\sf TS_2\mbox{-}  Correlated}$ as expected, but still significantly better than the independent Beta posterior version ${\sf TS_1\mbox{-} Beta}$. Therefore, significant improvements in performance due to correlated sampling feature of Algorithm~\ref{learn_algo_normal} compensate for the slight deterioration caused by boosting.

\section{Conclusion.} In this paper, we consider a combinatorial variant of the traditional multi-armed Bandit problem, 
\banditMNL\;and present a TS based policy for this problem. Focusing on designing a computationally efficient algorithm that facilitates theoretical analysis, we highlight several challenges involved in adaptive TS based approaches for the \banditMNL\;problem and discuss algorithm design choices to address them. To the best of our knowledge, the idea of correlated sampling for combinatorial arms is novel, and potentially useful for further combinatorial bandit problems. 

\begin{APPENDICES}

\section{Unbiased Estimate $\tilde{v}_{i,\ell}$ and Conjugate priors}
 We first prove that the estimate obtained from epoch based offerings, $\tilde{v}_{i,\ell}$ in Algorithm \ref{learn_algo} is unbiased estimate and is distributed geometrically with probability of success $\frac{1}{v_i+1}.$ Specifically, we have the following result.
\agrawalUnbiased*
\vspace{3pt}
\begin{proof}{Proof.}
We prove the result by computing the moment generating function, from which we can establish that $\tilde{v}_{i,\ell}$ is a geometric random variable with parameter $\frac{1}{1+v_i}$. Thereby also establishing that $\tilde{v}_{i,\ell}$  are unbiased estimators of $v_i$.  Specifically, we show the following result.

The moment generating function of estimate conditioned on $S_\ell$, $\hat{v}_{i}$, is given by,
\begin{eqnarray*}
\bb{E}\left(e^{\theta\tilde{v}_{i,\ell}} {\Big |} S_\ell\right) = \frac{1}{1-v_i(e^\theta-1)},\;\text{for all}\; \;\theta \leq \log{\frac{1+v_i}{v_i}},  \; \text{for all} \;\; i=1,\cdots,N.
\end{eqnarray*}
We focus on proving the above result. From \eqref{choice_probabilities}, we have that probability of no purchase event when assortment $S_\ell$ is offered is given by \[\textstyle p_0(S_\ell) = \displaystyle \frac{1}{1+\sum_{j\in S_\ell} v_j}.\] 
Let $n_\ell$ be the total number of offerings in epoch $\ell$ before a no purchased occurred, i.e., $n_\ell = |\ep{E}_\ell|-1$. 
Therefore, $n_\ell$ is a geometric random variable with probability of success $p_0(S_\ell)$. 
And, given any fixed value of $n_{\ell}$,
$\tilde{v}_{i,\ell}$ is a binomial random variable with $n_\ell$ trials and probability of success given by 
$$ \textstyle q_i(S_\ell)= \displaystyle \frac{v_i}{\sum_{j\in S_\ell} v_j}.$$ 
In the calculations below, for brevity we use $p_0$ and $q_i$ respectively to denote $p_0(S_\ell)$ and $q_i(S_\ell)$. Hence, we have 
\[\bb{E}\left( e^{\theta\tilde{v}_{i,\ell}}\right) = E_{n_\ell}\left\{\bb{E}\left( e^{\theta\tilde{v}_{i,\ell}}\,\middle| \,n_\ell\right)\right\}.\]
Since the moment generating function for a binomial random variable with parameters $n,p$ is $\left(pe^\theta + 1-p\right)^n$, we have
\[\bb{E}\left( e^{\theta\tilde{v}_{i,\ell}}\, \middle| \, n_\ell\right) =  E_{n_\ell}\left\{{\left(q_ie^\theta + 1-q_i\right)^{n_\ell}}\right\}.\]
For any $\alpha$, such that, $\alpha(1-p) < 1$ $n$ is a geometric random variable with parameter $p$, we have \[\bb{E}(\alpha^n) = \frac{p}{1-\alpha(1-p)}.\] Note that for all $\theta < \log{\frac{1+v_i}{v_i}}$, we have $\left(q_ie^\theta+(1-q_i)\right)(1-p_0) = (1-p_0) + p_0v_i(e^\theta-1) < 1.$
Therefore, we have 
\begin{eqnarray*}
\bb{E}\left( e^{\theta\tilde{v}_{i,\ell}}\right) = \displaystyle \frac{1}{1-v_i(e^\theta-1)}\;\text{for all} \; \theta < \log{\frac{1+v_i}{v_i}}. \text{\hfill $\Halmos$}
\end{eqnarray*} 
\end{proof}

Building on this result. We will prove Lemma \ref{conjugate_prior}  that helped construct Algorithm \ref{learn_algo}. Recall Lemma \ref{conjugate_prior},
\conjugateprior*
\begin{proof}{Proof.} The proof of the lemma follows from the following result on the probability density function of the random variable $X_{\alpha,\beta}$. Specifically, we have for any $x > 0$
\begin{eqnarray}\label{pdf}
f_{\alpha,\beta}(x) = \frac{1}{B(\alpha,\beta)} \left(\frac{1}{1+x}\right)^{\alpha+1} \left(\frac{x}{x+1}\right)^{\beta-1},
\end{eqnarray}
where $B(a,b) = \frac{\Gamma(a)\Gamma(b)}{\Gamma(a+b)}$ and $\Gamma(a)$ is the gamma function. Since we assume that the parameter $v_i$'s prior distribution is same as that of $X_{\alpha,\beta}$, we have from \eqref{pdf} and Lemma \ref{unbiased_estimate},  $$\bb{P}\left(v_i \big | \tilde{v}_{i,\ell} = m \right) \propto \left(\frac{1}{1+v_i}\right)^{\alpha+2} \left(\frac{v_i}{v_i+1}\right)^{\beta+m-1}. \Halmos$$
\end{proof}

Given the pdf of the posterior in \eqref{pdf}, it is possible to compute the mean and variance of the posterior distribution. We show that they have simple closed form expressions. Recall Lemma \ref{posterior_moment}. 
\posteriormoment*
\begin{proof}{Proof.} We prove the result by relating the mean of the posterior to the mean of the Beta distribution. 
Let $\hat{X} = \frac{1}{X}-1.$ From \eqref{pdf}, we have 
\begin{eqnarray*}
\bb{E}(\hat{X}) = \frac{1}{B(\alpha,\beta)}\int_{0}^{\infty} x\left(\frac{1}{1+x}\right)^{\alpha+1} \left(\frac{x}{x+1}\right)^{\beta-1}dx,
\end{eqnarray*}
Substituting $y = \frac{1}{1+x}$, we have 
\begin{eqnarray*}
\bb{E}(\hat{X}) = \frac{1}{B(\alpha,\beta)}\int_{0}^{1} y^{\alpha-2} (1-y)^{\beta}dx = \frac{B(\alpha-1,\beta+1)}{B(\alpha,\beta)} = \frac{\beta}{\alpha-1}.
\end{eqnarray*}
Similarly, we can derive the expression for the ${\sf Var}(\hat{X})$. 
\end{proof}

\section{Structural properties of the optimal revenue for the MNL model}\label{appsec:structprop}
Here, we prove the restricted monotonicity and Lipschitz property of the optimal revenue of the MNL model.
\structprop*
\begin{proof}{Proof.}
We will first prove the restricted monotonicity property and extend the analysis to prove the Lipschitz property. 

\medskip \noindent  {\bf Restricted Monotonicity.} We prove the result by first showing that for any $j\in S$, we have $R(S,{\mathbf{v}}^j) \geq R(S,\mathbf{v})$, where ${\mathbf{w}}^j$ is vector $\mathbf{v}$ with the $j^{th}$ component increased to $w_{j}$, i.e.  $v^j_i = v_i$ for all $i \neq j$ and $w^j_j = w_j$. We can use this result iteratively to argue that increasing each parameter of MNL to the highest possible value increases the value of $R(S,\mathbf{w})$ to complete the proof. 

If there exists $j \in S$ such that $r_j < R(S)$, then removing the product $j$ from assortment $S$ yields higher expected revenue contradicting the optimality of $S$. Therefore, we have 
$$r_j \geq R(S). ~\forall j \in S.$$ 
Multiplying by $({v}^{UCB}_j-{w}_j)(\sum_{i \in S/j } w_i + 1)$ on both sides of the above inequality and re-arranging terms, we can show that $R(S,{\mathbf{w}}^j) \geq R(S,\mathbf{w})$. 

\medskip \noindent  {\bf Lipschitz.} Following the above analysis, we define sets $\ep{I}(S^*)$ and $\ep{D}(S^*)$  as
\begin{eqnarray*}
\begin{aligned}
\ep{I}(S^*) &= \left\{i \middle | i \in S^* \;\text{and} \; v_{i} \geq w_i\right\}\\
\ep{D}(S^*) &= \left\{i \middle | i \in S^* \;\text{and} \; v_{i} < w_i\right\},
\end{aligned}
\end{eqnarray*}
and vector $\mb{u}$ as,
\begin{eqnarray*}
{u}_i = \left\{ 
\begin{array}{ll}
w_i \;& \text{if} \; \; i \in \ep{D}(S^*),\\
v_i \; & \text{otherwise} .
\end{array}\right.
\end{eqnarray*}
By construction of $\mb{u}$, we have $u_i \geq v_i$ and $u_i \geq w_i$ for all $i $. Therefore from the restricted monotonicity property, we have 
\begin{equation*}
\begin{aligned}
R(S^*,\mb{v}) - R(S^*,\mb{w}) &\leq R(S^*, \mb{u}) - R(S^*,\mb{w})\\
&\leq  \displaystyle \frac{\displaystyle\sum_{i \in S^*}r_i u_i}{1+\displaystyle\sum_{j \in S^*} u_j} - \frac{\displaystyle\sum_{i \in S^*}r_i w_i}{1+\displaystyle\sum_{j \in S^*} u_j},\\
&\leq  \frac{\displaystyle \sum_{i \in S^*} \left(u_i - w_i\right)}{1+\displaystyle\sum_{j \in S^*} u_j}.
\end{aligned}
\end{equation*}
The result follows from the fact that $u_i \geq v_i$ and $u_i \geq w_i$ for all $i \in S^*.$ \hfill $\Halmos.$
\end{proof}
\section{Bounds on the deviation of MNL Expected Revenue} \label{sec:lipschitz}
Here, we bound the difference between the expected revenues of the offer set $S_\ell$ corresponding to the sampled parameters, $\pmb{\mu}(\ell)$ and the true parameters, $\mb{v}$. In order to establish this bound, we will first present two concentration results. In the first result, utilizing the large deviation properties of Gaussian distribution, we show that over time, the posterior distributions concentrate around their means. The second result proves a Chernoff-like bound which suggests that the means of the posterior distribution concentrates around the true parameters. The proof of the second result is involved and hence, for ease of exposure, we defer the proof to Appendix \ref{sec:concentration}. 

\begin{lemma1}\label{lem:normal_concentration_inequality} ~For any $\ell \leq T$ and $i \in \{1, \cdots, N\}$, we have for any $r > 0$, 
$$\bb{P}\left( | \mu_{i}(\ell) -  \hat{v}_i(\ell) | > 4\hat{\sigma}_i(\ell) \sqrt{\log{rK}}\right) \leq \frac{1}{r^4K^3}, $$where $\hat{\sigma}_i(\ell) = \sqrt{\frac{50\hat{v}_i(\ell)(\hat{v}_i(\ell)+1)}{n_i(\ell)}}  + \frac{75\sqrt{\log{TK}}}{n_i(\ell)}$.
\end{lemma1}
\begin{proof}{Proof.~}
Note that we have $\mu_i(\ell) = \hat{v}_{i}(\ell) + \hat{\sigma}_i(\ell) \cdot \underset{j=1, \cdots , K}{\text{max}}\; \{ \theta^{(j)}(\ell)\}.$ Therefore, from union bound, we have, 
\begin{eqnarray*}
\begin{aligned}
\bb{P}\left\{ | \mu_{i}(\ell) -  \hat{v}_i(\ell) | > 4\hat{\sigma}_i(\ell) \sqrt{\log{rK}} \; \; \Big| \hat{v}_i(\ell)\right\} &= \bb{P}\left( \bigcup\limits_{j=1}^K \left\{\theta^j(\ell) >4 \sqrt{\log{rK}}  \right\}\right)\\
&\leq \sum_{j=1}^K \bb{P}\left( \theta^j(\ell) > 4\sqrt{\log{rK}} \right)\\
\end{aligned}
\end{eqnarray*}
The result follows from the above inequality and the following anti-concentration bound for the normal random variable $\theta^{(j)}(\ell)$ (see formula 7.1.13 in \cite{abramowitz1964handbook}).
$$\frac{1}{4\sqrt{\pi}}\cdot e^{-7z^2/2} < {\sf Pr} \left( |\theta^{(j)}(\ell)| > z \right) \leq \frac{1}{2} e^{-z^2/2}.$$ \hfill $\Halmos$
\end{proof}
\begin{lemma1}\label{multiplicative_chernoff_estimates}
~If $v_i \leq 1$ for all $i=1, \cdots, N$, then for any $m,\rho > 0$, $\ell \in \{1,2,\cdots\}$  and $i \in \{1, \cdots, N\}$ we have,
\begin{enumerate}
\vspace{5pt}
\item $\displaystyle \ep{P}\left(\left|\hat{v}_i(\ell) - {v_i}\right| > 4\sqrt{\frac{\hat{v}_{i}(\ell)(\hat{v}_i(\ell)+1)m\log{(\rho+1)}}{n_i(\ell)}} + \frac{24m \log{(\rho+1)}}{n_i(\ell)}\right)   \leq  \frac{5}{\rho^m}.$
\vspace{5pt}
\item $\displaystyle \ep{P}\left(\left|\hat{v}_i(\rho) -v_i \right | \geq  \sqrt{\frac{12{v_i}m\log{(\rho+1)}}{n_i(\ell)}} + \frac{24m\log{(\rho+1)}}{n_i(\ell)}\;\right)  \leq \frac{4}{\rho^m}$.
\end{enumerate}
\end{lemma1}
From Lemma \ref{structprop}, Lemma \ref{lem:normal_concentration_inequality} and Lemma \ref{multiplicative_chernoff_estimates}, we have the following result. 
\begin{lemma1}\label{lem:lipschitz}
~For any epoch $\ell$, if $S_\ell = \underset{S : |S| \leq K}{\text{arg~max}}\;{R}(S,\pmb{\mu}(\ell))$
\begin{eqnarray*}
\bb{E}\left\{(1+\sum_{j \in S_\ell} v_j) \left[R(S_\ell,\pmb{\mu}(\ell)) - R(S_\ell,\mb{v})\right]\right\} \leq \bb{E}\left[C_1\sum_{i \in S_\ell}\sqrt{\frac{v_i\log{TK}}{n_i(\ell)}} + C_2\frac{\log{TK}}{n_i(\ell)}\right],
\end{eqnarray*}
where $C_1$ and $C_2$ are absolute constants (independent of problem parameters).
\end{lemma1}


\section{Proof of Theorem 1}
\label{app:mainProof}
~

\medskip \noindent {\bf Notations.} For the sake of brevity, we introduce some notations. 
\begin{itemize}
\item For any assortment S, $V(S) \overset{\Delta}= \sum_{i\in S}v_i$
\vspace{3pt}
\item For any $\ell,\tau \leq L$, define $\Delta R_\ell$ and $\Delta R_{\ell, \tau}$ in the following manner 
\vspace{3pt}
\begin{eqnarray*}
\begin{aligned}
\Delta R_\ell & \overset{\Delta}{=} (1+V(S_\ell))\left[R(S_\ell,\pmb{\mu}(\ell)) -  R(S_\ell,\mb{v})\right]\\
\Delta R_{\ell,\tau}& \overset{\Delta}{=} (1+V(S_\tau))\left[R(S_\ell,\pmb{\mu}(\ell)) -  R(S_\ell,\pmb{\mu}(\tau))\right]
\end{aligned}
\end{eqnarray*}
\item Let $\ep{A}_0$ denote the complete set $\Omega$ and for all $\ell =1,\ldots,L$, define events $\ep{A}_{\ell}$  as
\begin{eqnarray*}
\begin{aligned}
\ep{A}_{\ell} &= \left\{ \left|\hat{v}_i(\ell) -v_i \right | \geq  \sqrt{\frac{24{v_i}\log{(\ell+1)}}{n_i(\ell)}} + \frac{48\log{(\ell+1)}}{n_i(\ell)}\;\; \text{for some $i = 1, \cdots, N$}\right\}
\end{aligned}
\end{eqnarray*}
 where $\hat{\sigma}_i(\ell) = \sqrt{\frac{50\hat{v}_i(\hat{v}_i+1)}{n_i}}  + \frac{75\sqrt{\log{TK}}}{n_i}$.

\item 
\begin{eqnarray}\label{eq:analysis_epoch}
\begin{aligned}
\ep{T} &= \left\{\ell : \mu_{i}(\ell) \geq v_i \; \; \text{for all} \; i \in S^*\right\},\\
{\sf succ}(\ell) &= \min\{\bar{\ell} \in \ep{T}: \bar{\ell} > \ell\}\\
\ep{E}^{\sf An}(\ell) &= \{\tau: \tau \in (\ell,{\sf succ}(\ell))\}\;\;\text{for all} \ell \in \ep{T}.
\end{aligned}
\end{eqnarray}
Here $\ep{T}$ is the set of ``optimistic'' epoch indices, i.e. when value of $\mu_i(\ell)$ is higher than the value of $v_i$ for all products $i$ in the optimal offer set $S^*$ and ${\sf succ}(\ell)$ denote the next epoch index after $\ell$ that is optimistic. $\ep{E}^{\sf An}(\ell)$ be the set of epoch indices's between an optimistic epoch, $\ell \in \ep{T}$ and the successive epoch. We will refer to $\ep{E}^{\sf An}(\ell)$ as the ``analysis epoch'' starting at $\ell$. To avoid confusion, we will refer to the epoch in which a selected offer set is offered until an outside option is preferred as ``algorithmic epoch.'' More specifically, for the rest of this proof, we will refer to $\ep{E}_\ell$ as $\ep{E}^{\sf Al}_\ell$. Note that the analysis epoch can contain one or more algorithmic epochs. 
\end{itemize}
\vspace{1pt}
\begin{eqnarray}\label{eq:regret_decomposition}
\begin{aligned}
Reg(T,\mb{v}):& = \bb{E}\left[\sum_{\ell=1}^{L}|\ep{E}_\ell| \left(R(S^*,\mb{v}) - R(S_\ell,\mb{v})\right)\right]\\
&= \underbrace{ \bb{E}\left[\sum_{\ell=1}^L |\ep{E}_\ell| \left(R(S^*, \mb{v}) - R(S_\ell, \pmb{\mu}(\ell)) \right)\right]}_{\Reg_1(T,\mb{v})} + \underbrace{ \bb{E}\left[\sum_{\ell=1}^L |\ep{E}_\ell| \left(R(S_\ell, \pmb{\mu}(\ell)) - R(S_\ell,\mb{v})\right)\right]}_{\Reg_2(T,\mb{v})}.
\end{aligned}
\end{eqnarray}
We will complete the proof by bounding the two terms in \eqref{eq:regret_decomposition}.
%

We first focus on bounding $Reg_2(T,\mb{v}).$

\medskip \noindent {\bf Bounding $Reg_2(T,\mb{v})$}:  We have, 
\begin{eqnarray*}
\bb{E}\left[ |\ep{E}^{\sf Al}_\ell| \left(R(S_\ell,\pmb{\mu}(\ell)) -  R(S_\ell,\mb{v})\right)\right] = \bb{E}\left[ \bb{E}\left(|\ep{E}^{\sf Al}_\ell|  \Big | S_\ell\right)\left(R(S_\ell,\pmb{\mu}(\ell)) -  R(S_\ell,\mb{v})\right)\right] ,
\end{eqnarray*}
and conditioned on the event $S_\ell = S$, the length of the $\ell^{th}$ algorithmic epoch, $|\ep{E}^{\sf Al}|$ is a geometric random variable with probability of success $p_0(S_\ell)$, where $$p_0(S_\ell) = \frac{1}{1+\sum_{j \in S_\ell}v_j}.$$ Therefore, it follows that 
\begin{eqnarray}\label{eq:epoch_length}
\bb{E}\left(|\ep{E}^{\sf Al}| \Big | S_\ell =S\right) = 1+V(S). 
\end{eqnarray}
Hence the second term in \eqref{eq:regret_decomposition} can be reformulated as 
\begin{eqnarray}\label{eq:ce_regret_2}
\begin{aligned}
Reg_2(T,\mb{v}) &=  \bb{E}\left\{\sum_{\ell=1}^L \Delta R_\ell\right\}.\\
\end{aligned}
\end{eqnarray}
Noting that $\ep{A}_\ell$ is a ``low probability'' event, we analyze the regret in two scenarios, one when $\ep{A}_\ell$ is true and another when $\ep{A}^c_\ell$ is true.   More specifically, 
\begin{eqnarray*}
\begin{aligned}
\bb{E}\left(\Delta R_\ell\right)= \bb{E}\left[\Delta R_\ell\cdot\bbm{1}(\ep{A}_{\ell-1}) + \Delta R_\ell\cdot\bbm{1}(\ep{A}^c_{\ell-1})\right],
\end{aligned}
\end{eqnarray*}
Using the fact that $R(S_\ell,\pmb{\mu}(\ell))$ and $R(S_\ell,\mb{v})$ are both bounded by one and $V(S_\ell) \leq K$, we have 
\begin{eqnarray*}
\begin{aligned}
\bb{E}\left(\Delta R_\ell\right) \leq (K+1)\bb{P}(\ep{A}_{\ell-1}) +\bb{E}\left[ \Delta R_\ell\cdot\bbm{1}(\ep{A}^c_{\ell-1})\right].
\end{aligned}
\end{eqnarray*}
Substituting $m=2$ and $\rho = \ell$ in Lemma \ref{multiplicative_chernoff_estimates},  we obtain that $\bb{P}(\ep{A}_{\ell-1}) \leq \frac{1}{\ell^2}$. Therefore, it follows that,
\begin{eqnarray}\label{eq:second_term_low_prob_bound}
\bb{E}\left\{\Delta R_\ell\right\} \leq \frac{K+1}{\ell^2} +\bb{E}\left[ \Delta R_\ell\cdot\bbm{1}(\ep{A}^c_{\ell-1})\right].
\end{eqnarray}
From Lemma \ref{structprop}, we have that 
\begin{eqnarray*}
R(S_\ell,\pmb{\mu}(\ell)) -  R(S_\ell,\mb{v}) \leq \frac{ \sum_{i \in S_\ell}| \mu_i(\ell) - v_i |}{1+\displaystyle \sum_{j \in S_\ell}v_j}.
\end{eqnarray*}
Therefore, from \eqref{eq:epoch_length} it follows that,
\begin{eqnarray*}
\bb{E}\left[ \Delta R_\ell \cdot \bbm{1}(\ep{A}^c_{\ell-1})\right] \leq \bb{E}\left[\sum_{i\in S_\ell} |\mu_i(\ell) - v_i| \cdot \bbm{1}(\ep{A}^c_{\ell-1})\right].
\end{eqnarray*}
From triangle inequality, we have 
\begin{eqnarray*}
\bb{E}\left[ \Delta R_\ell \cdot \bbm{1}(\ep{A}^c_{\ell-1})\right] \leq \bb{E}\left[\sum_{i\in S_\ell} |\mu_i(\ell) - \hat{v}_i(\ell)| \cdot \bbm{1}(\ep{A}^c_{\ell-1})\right] + \bb{E}\left[\sum_{i\in S_\ell} |\hat{v}_i(\ell) - v_i| \cdot \bbm{1}(\ep{A}^c_{\ell-1})\right],
\end{eqnarray*}
and from the definition of the event $\ep{A}^c_{\ell-1}$, it follows that, 
\begin{eqnarray}\label{eq:second_term_decompose}
\bb{E}\left[ \Delta R_\ell \cdot \bbm{1}(\ep{A}^c_{\ell-1})\right] \leq \bb{E}\left[\sum_{i\in S_\ell} |\mu_i(\ell) - \hat{v}_i(\ell)| \right] + \bb{E}\left[\sqrt{\frac{24{v_i}\log{(\ell+1)}}{n_i(\ell)}} + \frac{48\log{(\ell+1)}}{n_i(\ell)}\right].
\end{eqnarray}
We will now focus on bounding the first term in \eqref{eq:second_term_decompose}.  In Lemma \ref{lem:normal_concentration_inequality}, we show that for any $r >0$ and $i = 1, \cdots, N$,  we have, $$\bb{P}\left( | \mu_{i}(\ell) -  \hat{v}_i(\ell) | > 4\hat{\sigma}_i(\ell) \sqrt{\log{rK}}\right) \leq \frac{1}{r^4K^3}, $$where $\hat{\sigma}_i(\ell) = \sqrt{\frac{50\hat{v}_i(\hat{v}_i+1)}{n_i}}  + \frac{75\sqrt{\log{TK}}}{n_i}$. Since $S_\ell \subset \{1,\cdots, N\}$, we have for any $i \in S_\ell$ and $r > 0$, we have  
\begin{eqnarray}\label{eq:normal_conc_bound}
\begin{aligned}
\bb{P}\left( | \mu_{i}(\ell) -  \hat{v}_i(\ell) | > 4\hat{\sigma}_i(\ell) \sqrt{\log{rK}}\;\text{for any $i \in S_\ell$}\right) & \leq\bb{P}\left(\bigcup_{i=1}^N | \mu_{i}(\ell) -  \hat{v}_i(\ell) | > 4\hat{\sigma}_i(\ell) \sqrt{\log{rK}}\right),\\
&\\
&\leq \frac{N}{r^4K^3}.
\end{aligned}
\end{eqnarray}

Since $|\mu_{i}(\ell) -  \hat{v}_{i}(\ell)|$ is a non-negative random variable, we have  
\begin{eqnarray}\label{eq:geq_expectation_formula}
\begin{aligned}
 \bb{E}(|\mu_{i}(\ell) -  \hat{v}_{i}(\ell)|) &= \int_{0}^\infty \bb{P}\left\{ |\mu_{i}(\ell) -  \hat{v}_{i}(\ell) |\geq x \right\}dx, \\
& = \int_{0}^{4\hat{\sigma}_i(\ell)\sqrt{\log{TK}}} \bb{P}\left\{ |\mu_{i}(\ell) -  \hat{v}_{i}(\ell)|  \geq x \right\} dx + \int_{4\hat{\sigma}_i(\ell)\sqrt{\log{TK}}}^\infty \bb{P}\left\{ |\mu_{i}(\ell) -  \hat{v}_{i}(\ell)|  \geq x \right\} dx,\\
&\leq 4\hat{\sigma}_i(\ell)\sqrt{\log{TK}} + \sum_{r = T}^{\infty}\int_{4\hat{\sigma}_i(\ell)\sqrt{\log{rK}}}^{4\hat{\sigma}_i(\ell)\sqrt{\log{(r+1)K}}} \bb{P}\left\{ Y  \geq x \right\} dx, \\
&\\
& \overset{a}{\leq } 4\hat{\sigma}_i(\ell)\sqrt{\log{TK}} + \sum_{r=T}^\infty\frac{N\sqrt{\log{(rK+1)}} - N\sqrt{\log{rK}}}{r^4K^3},\\
& \leq 4\hat{\sigma}_i(\ell)\sqrt{\log{TK}} \;\; \text{for any $T \geq N$},
\end{aligned}
\end{eqnarray}
where the inequality (a) follows from \eqref{eq:normal_conc_bound}. From \eqref{eq:ce_regret_2}, \eqref{eq:second_term_low_prob_bound}, \eqref{eq:second_term_decompose} and Lemma \ref{multiplicative_chernoff_estimates}, we have,
$$Reg_2(T,\mb{v}) \leq C_1\bb{E}\left( \sum_{\ell=1}^L\sum_{i \in S_\ell}\sqrt{\frac{v_i\log{TK}}{n_i(\ell)}}\right) + C_2\bb{E}\left( \sum_{\ell=1}^L\sum_{i \in S_\ell}\frac{\log{TK}}{n_i(\ell)} \right),$$
where $C_1$ and $C_2$ are absolute constants. If $T_i$ denote the total number of epochs product $i$ is offered, then we have,
\begin{eqnarray}\label{eq:regret_2nd_term_jensen}
\begin{aligned}
Reg_2(T,\mb{v}) & \overset{(a)}{\le} C_2N\log^2{TK} + C_1\bb{E}\left(\sum_{i=1}^n  \sqrt{v_iT_i\log{TK}} \right), \\
& \overset{(b)}{\le} C_2N\log^2{TK} + C_1\sum_{i=1}^N  \sqrt{v_i\log{(TK)}\bb{E}(T_i)}.
\end{aligned}
\end{eqnarray}
Inequality (a) follows from the observation that $L \leq T$, $T_i \leq T$, $\displaystyle \sum_{n_i(\ell)=1}^{T_i} \frac{1}{\sqrt{n_i(\ell)}} \leq \sqrt{T_i}$ and $\displaystyle \sum_{n_i(\ell)=1}^{T_i} \frac{1}{{n_i(\ell)}} \leq \log{T_i}$, while Inequality (b) follows from Jensen's inequality. 

 For any realization of $L$, $\ep{E}^{\sf Al}_\ell$, $T_i$, and $S_\ell$ in Algorithm \ref{learn_algo}, we have the following relation $\sum_{\ell=1}^L |\ep{E}^{\sf Al}_\ell|  \leq T$. Hence, we have $\bb{E}\left(\sum_{\ell=1}^L |\ep{E}^{\sf Al}_\ell|\right)  \leq T.$
Let $\mathcal{S}$ denote the filtration corresponding to the offered assortments $S_1,\cdots,S_L$, then by law of total expectation, we have, 
\begin{eqnarray*}
\begin{aligned}
\bb{E}\left(\sum_{\ell=1}^L |\ep{E}^{\sf Al}_\ell|\right) &= \bb{E}\left\{\sum_{\ell=1}^L E_{\mathcal{S}}\left( |\ep{E}^{\sf Al}_\ell|\right)\right\}= \bb{E}\left\{\sum_{\ell=1}^L1+\sum_{i\in S_\ell} v_i\right\}, \\
&=  \bb{E}\left\{L+\sum_{i=1}^n  v_i T_i\right\} = \bb{E}\{L\}+\sum_{i=1}^n  v_i \bb{E}(T_i).
\end{aligned}
\end{eqnarray*}
Therefore, it follows that  
\begin{eqnarray*}
\sum v_i\bb{E}(T_i) \leq T.
\end{eqnarray*}
To obtain the worst case upper bound, we maximize the bound in equation \eqref{eq:regret_2nd_term_jensen} subject to the above condition and hence, we have 
\begin{eqnarray}\label{eq:second_term_bound}
Reg_2(T,\mb{v})  \leq   C_1 \sqrt{NT\log{TK}}   + C_2 N\log^2{TK}).
\end{eqnarray}
We will now focus on the first term in \eqref{eq:regret_decomposition}.

\medskip \noindent {\bf Bounding $Reg_1(T,\mb{v})$:}
Recall, $\ep{T}$ is the set of optimistic epoch and the sanalysis epoch $\ep{E}^{\sf An}(\ell)$ is the set of non-optimistic epochs between $\ell^{th}$ epoch and the subsequent optimistic epoch. Therefore, we can reformulate $Reg_1(T,\mb{v})$ as, 
$$Reg_1(T,\mb{v}) = \bb{E}\left[\sum_{\ell =1 }^{L}   \bbm{1}(\ell \in \ep{T}) \cdot \sum_{\tau \in \ep{E}^{\sf An}(\ell)}  |\ep{E}^{\sf Al}_\tau| \left(R(S^*,\mb{v}) - R(S_\tau,\pmb{\mu}(\tau))\right)\right].$$
Note that for any $\ell$, by algorithm design we have that $S_\ell$ is the optimal set when the MNL parameters are given by $\pmb{\mu}(\ell)$, i.e., $R(S_\ell,\pmb{\mu}(\ell)) \geq R(S^*,\pmb{\mu}(\ell))$. From the restricted monotonicity property (see Lemma \ref{structprop}), for any $\ell \in \ep{T}$, we have $R(S^*,\pmb{\mu}(\ell)) \geq R(S^*,\mb{v})$. Therefore, it follows that, 
\begin{eqnarray*}\label{eq:first_term}
Reg_1(T,\mb{v}) \leq \bb{E}\left[\sum_{\ell =1 }^{L}   \bbm{1}(\ell \in \ep{T}) \cdot \sum_{\tau \in \ep{E}^{\sf An}(\ell)}  |\ep{E}^{\sf Al}_\tau| \left(R(S_\ell,\pmb{\mu}(\ell)) - R(S_\tau,\pmb{\mu}(\tau))\right)\right].
\end{eqnarray*}
Observe that by design for any $t$, $R(S_\tau,\pmb{\mu}(\tau))\geq R(S,\pmb{\mu}(\tau))$ for any assortment $S$. Therefore, we have for any $\tau$, we have $R(S_\tau,\pmb{\mu}(\tau))\geq R(S_\ell,\pmb{\mu}(\tau))$. From \eqref{eq:epoch_length} we have,
\begin{eqnarray} \label{eq:first_term_decompose}
\begin{aligned}
Reg_1(T,\mb{v}) \leq \bb{E}\left[ \sum_{\ell = 1}^L  \bbm{1}(\ell \in \ep{T}) \cdot\sum_{\tau \in \ep{E}^{\sf An}(\ell)}  \Delta R_{\ell,\tau}\right].
\end{aligned}
\end{eqnarray}
Following the approach of bounding $Reg_2(T,\mb{v})$,  we analyze the first term, $Reg_1(T,\mb{v})$ in two scenarios, one when $\ep{A}_\ell$ is true and another when $\ep{A}^c_\ell$ is true.   More specifically, 
\begin{eqnarray*}
\begin{aligned}
\bb{E}\left(\sum_{\tau \in \ep{E}^{\sf An}(\ell)}\Delta R_{\ell,\tau}\right)= \bb{E}\left[\sum_{\tau \in \ep{E}^{\sf An}(\ell)}\Delta R_{\ell,\tau}\cdot\bbm{1}(\ep{A}_{\ell-1}) + \Delta R_{\ell,\tau}\cdot\bbm{1}(\ep{A}^c_{\ell-1})\right].
\end{aligned}
\end{eqnarray*}
Adding and subtracting $R(S_\ell,\mb{v})$, from triangle inequality and Lemma \ref{structprop}, we obtain
\begin{eqnarray*}
R(S_\ell,\pmb{\mu}(\ell)) -  R(S_\ell,\pmb{\mu}(\tau)) \leq \frac{ \sum_{i \in S_\ell}| \mu_i(\ell) - v_i | + | \mu_i(\tau) - v_i |}{1+V(S_\ell)}.
\end{eqnarray*}
Using the fact that $R(S_\ell,\pmb{\mu}(\ell))$ and $R(S_\ell,\pmb{\mu}(\tau))$ are both bounded by one and $V(S_\tau) \leq K$, we have 
\begin{eqnarray}\label{eq:first_term_ind_regret}
\bb{E}\left(\sum_{\tau \in \ep{E}^{\sf An}(\ell)}\Delta R_{\ell,\tau}\right) \leq (K+1)\bb{E}\left[|\ep{E}^{\sf An}(\ell)|\cdot\bbm{1}(\ep{A}_{\ell-1}) + \frac{ \bbm{1}(\ep{A}^c_{\ell-1}) }{1+V(S_\ell)}\sum_{\tau \in \ep{E}^{\sf An}(\ell)}\sum_{i \in S_\ell}| \mu_i(\ell) - v_i | + | \mu_i(\tau) - v_i |\right].\;\;\;\;\;\;\;\;\;\;\;
\end{eqnarray}
Following the approach of Bounding $Reg_1(T,\mb{v})$, specifically along the lines of  \eqref{eq:ce_regret_2}, \eqref{eq:second_term_low_prob_bound}, \eqref{eq:second_term_decompose} and \eqref{eq:geq_expectation_formula}, we can show that
$$\hspace{-9mm} \bb{E}\left[\sum_{\tau \in \ep{E}^{\sf An}(\ell)}\frac{ \sum_{i \in S_\ell} | \mu_i(\ell) - v_i | + | \mu_i(\tau) - v_i |}{1+V(S_\ell)}\cdot\bbm{1}(\ep{A}^c_{\ell-1})\right] \leq \bb{E}\left[\frac{|\ep{E}^{\sf An}(\ell)|}{1+V(S_\ell)}\sum_{i \in S_\ell} \left(C_1\sqrt{\frac{v_i\log{TK}}{n_i(\ell)}} + C_2 \frac{\log{TK}}{n_i(\ell)}\right)\right],$$
where $C_1$ and $C_2$ are constants. 
Hence, from \eqref{eq:first_term_decompose} and \eqref{eq:first_term_ind_regret}, we have 
\begin{eqnarray}\label{eq:first_term_final_decompose}
\begin{aligned}
\frac{Reg_1(T,\mb{v})}{K+1} &\leq \bb{E}\left[\sum_{\ell=1}^L|\ep{E}^{\sf An}(\ell)|\cdot\bbm{1}(\ep{A}_{\ell-1}) + \frac{|\ep{E}^{\sf An}(\ell)|}{1+V(S_\ell)}\sum_{i \in S_\ell} \left(C_1\sqrt{\frac{v_i\log{TK}}{n_i(\ell)}} + C_2 \frac{\log{TK}}{n_i(\ell)}\right) \right].\\
\end{aligned}
\end{eqnarray}
We bound each of term in the above expression to complete the proof. We have by Cauchy-Schwartz inequality, 
\begin{eqnarray*}
\begin{aligned}
\bb{E}\left[|\ep{E}^{\sf An}(\ell)|\cdot\bbm{1}(\ep{A}_{\ell-1})\right] &\leq \bb{E}^{1/2}\left( |\ep{E}^{\sf An}(\ell)|^2\right)\cdot \bb{P}^{1/2}\left( \ep{A}_{\ell-1}\right).
\end{aligned}
\end{eqnarray*}
Substituting $m=2$ and $\rho = \ell$ in Lemma \ref{multiplicative_chernoff_estimates},  we obtain that $\bb{P}(\ep{A}_{\ell-1}) \leq \frac{1}{\ell^2}$. In Lemma \ref{epoch_length_analysis}, we show that $\bb{E}^{1/2}\left[ \left|\ep{E}^{\sf An}(\tau) \right |^2  \right] \leq \frac{e^{12}}{K}+30^{1/2}.$  Therefore, we have 
\begin{eqnarray}\label{eq:low_prob}
\bb{E}\left[ \sum_{\ell=1}^L |\ep{E}^{\sf An}(\ell)| \cdot I(\ep{A}_{\ell-1})\right] \leq \frac{e^{13}}{K}. 
\end{eqnarray}
Now we bound the second term in \eqref{eq:first_term_final_decompose}. For notational brevity, let 
\begin{eqnarray*}
\begin{aligned}
\delta_i(\ell) &= \frac{C_1}{1+V(S_\ell)}\sum_{i \in S_\ell} \sqrt{\frac{v_i\log{TK}}{n_i(\ell)}}, \\
\Delta_i(\ell) & = \frac{C_2}{1+V(S_\ell)}\sum_{i \in S_\ell} \frac{\log{TK}}{n_i(\ell)}.
\end{aligned}
\end{eqnarray*}
From Cauchy-Schwartz inequality, we have
\begin{eqnarray}\label{eq:first_term_regret_maj}
\hspace{-9mm}\sum_{\ell=1}^L{|\ep{E}^{\sf An}(\ell)|} \left(\delta_i(\ell) + \Delta_i(\ell)\right) \leq \left({\sum_{\ell=1}^L|\ep{E}^{\sf An}(\ell)|^2}\right)^{1/2}\cdot \left[\left(\sum_{\ell=1}^L \delta^2_{i}(\ell)\right)^{1/2} + \left(\sum_{\ell=1}^L\Delta^2_i(\ell)\right)^{1/2}\right].\end{eqnarray}
 Again applying Cauchy-Schwartz on the summation $\sum_{i \in S_\ell} \sqrt{v_i}\sqrt{\frac{\log{TK}}{n_i(\ell)}}$, we have 
 \begin{eqnarray*}
 \begin{aligned}
 \delta^2_i(\ell) &\leq \frac{C_1^2V(S_\ell)}{(1+V(S_\ell))^2} \cdot \sum_{i\in S_\ell}\frac{\log{TK}}{n_i(\ell)}, \\
 &\leq C_1^2\sum_{i \in S_\ell} \frac{\log{TK}}{n_i(\ell)}.
 \end{aligned}
 \end{eqnarray*}
Let $T_i$ denote the total number of epochs product $i$ is offered, then we have, $$\sum_{\ell=1}^L \sum_{i\in S_\ell}\frac{\log{TK}}{n_i(\ell)} = \sum_{i=1}^N  \sum_{n_i(\ell)=1}^{T_i} \frac{\log{TK}}{n_i (\ell)} \leq N \log{TK} \cdot \log{T}.$$
From Lemma \ref{epoch_length_analysis} and preceding two equations, it follows that
\begin{eqnarray*}
\bb{E}\left[\left({\sum_{\ell=1}^L|\ep{E}^{\sf An}(\ell)|^2}\right)^{1/2}\cdot \left(\sum_{\ell=1}^L \delta^2_{i}(\ell)\right)^{1/2}\right]  \leq \frac{C_1e^{13}\sqrt{NT} \log{TK}}{K}.
\end{eqnarray*}
Noting that $\sum_{\ell=1}^L \Delta^2_i(\ell) \leq C_2\log^2{TK}$, we have from Lemma \ref{epoch_length_analysis}
\begin{eqnarray*}
\bb{E}\left[\left({\sum_{\ell=1}^L|\ep{E}^{\sf An}(\ell)|^2}\right)^{1/2}\cdot \left(\sum_{\ell=1}^L \delta^2_{i}(\ell)\right)^{1/2}\right]  \leq \frac{C_2e^{13}\sqrt{T} \log{TK}}{K}.
\end{eqnarray*}
Hence, from the preceding two results and from \eqref{eq:first_term_final_decompose}, from \eqref{eq:low_prob} and \eqref{eq:first_term_regret_maj}, we have 
\begin{eqnarray}\label{eq:first_term_bound}
Reg_1(T,\mb{v}) \leq C \sqrt{NT}\log{TK}, 
\end{eqnarray}
where $C$ is a constant.  The result follows from \eqref{eq:first_term_bound} and \eqref{eq:second_term_bound}. \hfill


\subsection{Bounding the analysis epoch length}
\label{app:analEpochLength}
Here, we prove that the expected length (and higher moments) of the analysis epoch (see \ref{eq:analysis_epoch}) is bounded by a constant. Specifically, we have the following result. 
\vspace{-15pt}
\optimism*

\begin{proof}{Proof.} For notational brevity, we introduce some notation. 

\medskip \noindent {\bf Notation.}
\begin{itemize}
\item $n_i(\ell)$ denote the number of epochs product $i$ has been offered until epoch $\ell$ (including epoch $\ell$) in Algorithm \ref{learn_algo_normal}.
\item Let $\hat{v}_{i}(\ell)$ denote the value of $\hat{v}_i$ after epoch $\ell$ . 
\item $$r = \lfloor (q+1)^{1/p} \rfloor,$$ $$z =  \sqrt{\log{(rK+1)}},$$ and for each $i=1,\cdots,N$, $$\hat{\sigma}_i(\ell) = 4\sqrt{\frac{m\hat{v}_i(\ell) (\hat{v}_i(\ell)+1)}{n_i(\ell)}}  + \frac{24m\sqrt{\log{TK}}}{n_i(\ell)}.$$
\item Define events, 
\begin{eqnarray}\label{eq:events_analysis_epoch_proof}
\begin{aligned}
A_\ell &= \left\{ \mu_i(\ell) \geq \hat{v}_i(\ell) +{z}{\hat{\sigma}_i(\ell)} \;\;\text{for all}\; i \in S^*\right\},\\
B_\ell&= \left\{ \hat{v}_i(\ell) +{z}{\hat{\sigma}_i(\ell)}  \geq v_i  \;\;\text{for all}\; i \in S^* \right\},\\
\ep{B}_\tau & = \bigcap \limits_{\ell = \tau+1}^{\tau+r} B_\ell.
\end{aligned}
\end{eqnarray}
\end{itemize}
We have,
\begin{eqnarray*}
\bb{P}\left\{ \left|\ep{E}^{\sf An}(\tau) \right |^p  < q +1 \right\} = \bb{P}\left\{ \left|\ep{E}(\tau) \right |  \leq r \right\}.
\end{eqnarray*}
By definition, length of the analysis epoch, $\ep{E}^{\sf An}(\tau)$ less than $r$, implies that one of the algorithm epochs from $\tau+1, \cdots, \tau+{r}$ is optimistic. Hence, we have,   
\begin{eqnarray*}
\begin{aligned}
\bb{P}\left\{ \left|\ep{E}^{\sf An}(\tau) \right |  < r \right\} & = \bb{P}\left( \Big\{\left\{ \mu_i(\ell) \geq v_i  \; \text{for all}  \; i \in S^* \right\} \; \text{for some} \; \ell \in (\tau, \tau + r] \Big\} \right),\\
& \hspace{-15mm} \geq \bb{P}\left(\Big\{\left\{\mu_i(\ell) \geq \hat{v}_i(\ell) + z \hat{\sigma}_i(\ell) \geq v_i  \; \text{for all}  \; i \in S^* \right\}\; \text{for some} \; \ell \in (\tau, \tau + r]\Big\} \right).
\end{aligned}
\end{eqnarray*}
From \eqref{eq:events_analysis_epoch_proof}, we have, 
\begin{eqnarray}\label{eq:complement_prob}
\begin{aligned}
\bb{P}\left\{ \left|\ep{E}^{\sf An}(\tau) \right |  < r \right\} &\geq \bb{P} \left(\bigcup\limits_{\ell=\tau+1}^{\tau+r} A_\ell \cap B_\ell \right),\\
& = 1-\bb{P}\left( \bigcap\limits_{\ell=\tau+1}^{\tau+r} A^c_\ell \cup B^c_\ell \right).
\end{aligned}
\end{eqnarray}
We will now focus on the term, $\bb{P}\left( \bigcap\limits_{\ell=\tau+1}^{\tau+r} A^c_\ell \cup B^c_\ell \right)$, 
\begin{eqnarray}\label{eq:conditional_expression}
\begin{aligned}
\hspace{-5mm}\bb{P}\left( \bigcap\limits_{\ell=\tau+1}^{\tau+r} A^c_\ell \cup B^c_\ell\right) &= \bb{P}\left(\left\{\bigcap\limits_{\ell=\tau+1}^{\tau+r} A^c_\ell \cup B^c_\ell\right\} \cap   \ep{B}_\tau  \right)  + \bb{P}\left(\left\{\bigcap\limits_{\ell=\tau+1}^{\tau+r} A^c_\ell \cup B^c_\ell\right\} \cap   \ep{B}^c_\tau \right),\\
&\leq\bb{P} \left(\bigcap\limits_{\ell=\tau+1}^{\tau+r}A^c_\ell  \right) + \bb{P}(\ep{B}^c_\tau ),\\
& \leq \bb{P} \left(\bigcap\limits_{\ell=\tau+1}^{\tau+r}A^c_\ell \right) \;+\; \sum_{\ell=\tau+1}^{\tau+r}\bb{P}({B}^c_\ell  ),\\
\end{aligned}
\end{eqnarray}
where the inequality follows from union bound. Note that, 
\begin{eqnarray}\label{eq:union_bound_martingale}
\begin{aligned}
\bb{P}(B^c_\ell) &= \bb{P}\left(\bigcup\limits_{i\in S^*}\left\{\hat{v}_i(\ell) +{z}{\hat{\sigma}_i(\ell)}  < v_i  \right\} \right),\\
& \leq \sum_{i \in S^*} \bb{P}\left(\hat{v}_i(\ell) +{z}{\hat{\sigma}_i(\ell)}  < v_i   \right).
\end{aligned}
\end{eqnarray}
Since $r$ is trivially less than $T$, we have $ rK + 1 \leq TK$, we have $\sqrt{\log{(rK+1)}\cdot \log{TK}}\geq \log{(rK+1)}$ and therefore it follows that,
$$\bb{P}\left(\hat{v}_i(\ell) +{z}{\hat{\sigma}_i(\ell)} < v_i   \right) \leq \bb{P}\left(\hat{v}_i(\ell) +4\sqrt{\frac{m\hat{v}_i(\ell) (\hat{v}_i(\ell)+1)\log\left(rK+1\right)}{n_i(\ell)}}  + \frac{24m\log{\left(rK+1\right)}}{n_i(\ell)}  < v_i   \right).$$
 Substituting $m=3.1$ and $\rho = rK$ in Lemma \ref{multiplicative_chernoff_estimates}, we obtain, 
\begin{eqnarray}\label{eq:low_prob_case}
\bb{P}\left(\hat{v}_i(\ell) +{z}{\hat{\sigma}_i(\ell)}  < v_i   \right) \leq \frac{1}{\;(rK)^{3.1}}.
\end{eqnarray}
From \eqref{eq:union_bound_martingale} and \eqref{eq:low_prob_case}, we obtain, 
\begin{eqnarray}\label{eq:low_prob_aggr}
\begin{aligned}
\bb{P}(B^c_\ell) &\leq \frac{1}{r^{3.1}K^{2.1}}, \;\;\text{and}\\
\sum_{\ell=\tau+1}^{\tau+r}\bb{P}(B^c_\ell) & \leq \frac{1}{(rK)^{2.1}}.
\end{aligned}
\end{eqnarray}
We will now use the tail bounds for Gaussian random variables to bound the probability $\bb{P}(A^c_\ell)$. For any Gaussian random variable, $Z$ with mean $\mu$ and standard deviation $\sigma$, we have, 
$$Pr(Z > \mu + x\sigma) \geq \frac{1}{\sqrt{2\pi}}\frac{x}{x^2+1}e^{-x^2/2}.$$
Note that by construction of $\mu_i(\ell)$ in Algorithm \ref{learn_algo_normal}. We have,
$$\bb{P}\left(\bigcap\limits_{\ell=\tau+1}^{\tau+r}A^c_\ell \right) = \bb{P}\Big(\theta^{(j)}(\ell) \leq z \;\; \text{for all} \; \ell \in (\tau,\tau+r] \;\text{and for all} \; j =1,\cdots,K \Big).$$
Since $\theta^{(j)}(\ell), \; j= 1,\cdots,K, \; \ell = \tau+1, \cdots, \tau+r$ are independently sampled from the distribution, $\ep{N}\left(0,1\right)$, we have,
\begin{eqnarray}\label{eq:prod}
\begin{aligned}
\bb{P}\left\{\bigcap\limits_{\ell=\tau+1}^{\tau+r}A^c_\ell \right\} & \leq \left[1- \left(\frac{1}{\sqrt{2\pi}}\frac{\sqrt{\log{(rK+1)}}}{\log{(rK+1)}+1}\cdot\frac{1}{\sqrt{rK+1}}\right)\right]^{rK} \\
& \leq \exp\left( - \frac{r^{1/2}}{\sqrt{2\pi}}\frac{2\sqrt{\log{(rK+1)}}}{4\log{(rK+1)}+1} \right)\\
& \leq \frac{1}{(rK)^{2.2}} \;\;\text{for any} \; r \geq \frac{e^{12}}{K}.
\end{aligned}
\end{eqnarray}
From \eqref{eq:complement_prob}, \eqref{eq:conditional_expression}, \eqref{eq:low_prob_aggr} and \eqref{eq:prod}, we have that,   
$$\bb{P}\left\{ \left|\ep{E}^{\sf An}(\tau) \right |  < r \right\} \geq 1-\frac{1}{(rK)^{2.1}} - \frac{1}{(rK)^{2.2}} \;\; \text{for any} \; r \geq \frac{e^{12}}{K}.$$
From definition  $r \geq (q+1)^{1/p}-1$, we obtain
$$\bb{P}\left\{ \left|\ep{E}^{\sf An}(\tau) \right |^p  < q+1 \right\} \geq 1-\frac{1}{(q+1)^{2.1/p}-1} - \frac{1}{(q+1)^{2.2/p}-1} \;\; \text{for any} \; q \geq \left(\frac{e^{12}}{K} + 1\right)^p.$$
Therefore, we have, 
\begin{eqnarray*}
\begin{aligned}
\bb{E}\left[ \left|\ep{E}^{\sf An}(\tau) \right |^p  \right] &= \sum_{q = 0}^\infty \bb{P}\left\{ \left|\ep{E}(\tau) \right |^p  \geq \ell \right\},\\
&\leq \left(\frac{e^{12}}{K} + 1\right)^p + \sum_{q = \frac{e^{12p}}{K^p}}^\infty \bb{P}\left\{ \left|\ep{E}(\tau) \right |^p  \geq \ell \right\},\\
&\leq e^{12p} + \displaystyle \sum_{q = \frac{e^{12p}}{K^p}}^\infty \frac{1}{\ell^{2.1/p}} + \frac{1}{\ell^{2.2/p}} \leq \left(\frac{e^{12}}{K} + 1\right)^p+30.
\end{aligned} 
\end{eqnarray*}
The result follows from the above inequality. \hfill $\Halmos.$
\end{proof}


\subsection{Some concentration bounds}\label{sec:concentration}
In this section, we prove bounds on how fast our estimate $\hat{v}_i$ converges to the true mean. For the rest of this section, we assume that $\hat{v}_i(\ell)$ and $n_i(\ell)$ are the values of $\hat{v}_i$ and $n_i$ in Algorithm \ref{learn_algo_normal} before the beginning of epoch $\ell$. The concentration bounds we prove in the section are similar to Chernoff bounds, but for the fact that $n_i(\ell)$ is a random variable and $\hat{v}_i(\ell)$ is the mean of random number of i.i.d samples. Hence, we use a self-normalized martingale technique to derive concentration bounds. Specifically, we have,

\begin{theor1}\label{multiplicative_chernoff_geometric}
~Let $\delta_i$, $i =1,\cdots, N$ be arbitrary random variables. If $v_i \leq 1,$ for all $i=1, \cdots, N$, then we have, for all $i = 1, \cdots, N$, 
\begin{enumerate}
\item
\begin{eqnarray*}
Pr\left( \hat{v}_{i}(\ell) > (1+\delta_i) v_i\right) \leq 
\bb{E}^{\frac{1}{2}}\left[\exp{\left(-\frac{v_i \delta_i^2n_i(\ell)}{2(1+\delta_i)(1+v_i)^2}\right)} \right]\; ,
\end{eqnarray*}
and 
\item \begin{eqnarray*}
Pr\left(\hat{v}_{i}(\ell) < (1-\delta_i) v_i  \right) \leq 
\bb{E}^{\frac{1}{2}}\left[\exp{\left(- \frac{v_i\delta_i^2 n_i(\ell)}{6 (1+v_i)^2}\left(3 - \frac{2\delta_iv_i}{1+v_i}\right)\right)}\right]. 
\end{eqnarray*}
\end{enumerate}
\end{theor1}
\begin{proof}{Proof.} Fix $i$.  We have $$\hat{v}_i(\ell) = \frac{1}{n_i(\ell)}\sum_{\tau=1}^{\ell} \tilde{v}_{i,\tau} \bbm{1}(i \in S_\tau).$$
Therefore, bounding $Pr\left( \hat{v}_{i}(\ell) > (1+\delta_i) v_i\right)$ and $Pr\left( \hat{v}_{i}(\ell) < (1-\delta_i) v_i\right)$ is equivalent to bounding $Pr\left( \sum_{\tau=1}^\ell \tilde{v}_{i,\tau} \bbm{1}(i \in S_\tau) > (1+\delta) v_in_i(\ell)\right)$ and $Pr\left( \sum_{\tau=1}^\ell \tilde{v}_{i,\tau} \bbm{1}(i \in S_\tau) < (1-\delta) v_in_i(\ell)\right)$. We will bound the first term and then follow a similar approach for bounding the second term  to complete the proof. 
\vspace{4pt}
\subsection*{ Bounding ${Pr\left(\hat{v}_{i}(\ell) > (1+\delta_i) v_i \right)}$:} 
From Markov Inequality,   we have for any $\lambda > 0$, 
\begin{eqnarray}\label{eq:lhs_MI}
\begin{aligned}
Pr\left(\sum_{\tau=1}^\ell \tilde{v}_{i,\tau} \bbm{1}(i \in S_\tau) > (1+\delta_i) v_i n_i(\ell)\right) &= Pr \left\{ \exp\left(\lambda \sum_{\tau=1}^\ell \tilde{v}_{i,\tau} \bbm{1}(i \in S_\tau) \right)> \exp\left(\lambda (1+\delta_i) v_i n_i(\ell)\right)\right\},\\
& = Pr \left\{ \exp\left(\lambda \sum_{\tau=1}^\ell \tilde{v}_{i,\tau} \bbm{1}(i \in S_\tau) - \lambda (1+\delta_i) v_in_{i}(\ell)\right) > 1\right\}, \;\;\;\;\;\;\;\;\;\;\;\;\\
& \leq \bb{E} \left[\exp\left(\lambda \sum_{\tau=1}^\ell \tilde{v}_{i,\tau} \bbm{1}(i \in S_\tau) - \lambda(1+\delta_i) v_in_i(\ell)\right)\right].
\end{aligned}
\end{eqnarray}
For notational brevity, denote $f(\lambda,v_i)$ by the function, 
$$f(\lambda,v_i) = -\frac{\log\left({1-v_i(e^{2\lambda}-1)}\right)}{2} .$$
We have, 
\begin{eqnarray}\label{eq:lhs_neg_corr}
\begin{aligned}
&\bb{E} \left[\exp\left(\lambda \sum_{\tau=1}^\ell \tilde{v}_{i,\tau} \bbm{1}(i \in S_\tau) - \lambda(1+\delta_i) v_in_i(\ell)\right)\right]\\
 &= \bb{E} \left[\exp\left(\sum_{\tau = 1}^\ell (\lambda\tilde{v}_{i,\tau} -f(\lambda,v_i))\cdot \bbm{1}(i \in S_\tau)\right) \cdot \exp\Big(-\lambda (1+\delta_i)v_i (1-f(\lambda,v_i))  n_i(\ell)\Big)\right],\\
&\leq \bb{E}^{\frac{1}{2}} \left[\exp\left(\sum_{\tau = 1}^\ell (2\lambda\tilde{v}_{i,\tau} - 2f(\lambda,v_i))\cdot \bbm{1}(i \in S_\tau)\right) \right] \cdot \bb{E}^{\frac{1}{2}}\left[\exp\Big(-2\lambda (1+\delta_i)v_i (1-f(\lambda,v_i)) n_i(\ell)\Big)\right],
\end{aligned}
\end{eqnarray}
where the above inequality follows from Cauchy-Schwartz inequality. Let $\ep{F}_\tau$ be the filtration corresponding to the history until epoch $\tau$. Note that for any $\tau$, $\bbm{1}(i \in S_\tau)$ conditioned on $F_{\tau}$ is a constant and $\{\tilde{v}_{i,\tau} | \ep{F}_\tau\}$ is a geometric random variable. From the proof of Lemma \ref{unbiased_estimate}, for all $\tau \geq 1$ and for any $0<\lambda <\frac{1}{2} \log{\frac{1+v_i}{v_i}},$ we have, $$\bb{E}\left(e^{2\lambda \tilde{v}_{i,\tau}\bbm{1}(i \in S_\tau)}\middle | \ep{F}_{\tau}\right) = \left(\frac{1}{1-v_i(e^{2\lambda}-1)}\right)^{\bbm{1}(i \in S_\tau)}.$$
Therefore, it follows that 
\begin{eqnarray}\label{eq:lhs_mgf_ineq}
\bb{E}\left(e^{(2\lambda \tilde{v}_{i,\tau}- 2f(\lambda,v_i)) \cdot \bbm{1}(i \in S_\tau)}\middle | \ep{F}_{\tau}\right) \leq 1,
\end{eqnarray}
and 
\begin{eqnarray*}
\begin{aligned}
\hspace{-10mm}\bb{E} \left[ \exp\left(\sum_{\tau = 1}^\ell (2\lambda \tilde{v}_{i,\tau} - 2f(\lambda,v_i))\cdot \bbm{1}(i \in S_\tau)\right) \right] &= \bb{E} \left[\bb{E}\left\{\exp\left((2\lambda\tilde{v}_{i,\tau} -2f(\lambda,v_i))\cdot \bbm{1}(i \in S_\tau)\right)  \middle | \ep{F}_\ell\right\}\right]\\
&\hspace{-25mm}= \bb{E} \left[\prod_{\tau=1}^{\ell-1}\exp\left( (2\lambda\tilde{v}_{i,\tau} -2f(\lambda,v_i))\cdot \bbm{1}(i \in S_\tau)\right)\cdot \bb{E}\left(e^{(2\lambda \tilde{v}_{i,\ell}- 2f(\lambda,v_i)) \cdot \bbm{1}(i \in S_\ell)}\middle | \ep{F}_{\ell}\right)\right]\\
&\hspace{-25mm} \leq \bb{E} \left[\prod_{\tau=1}^{\ell-1}\exp\left((2\lambda\tilde{v}_{i,\tau} -2f(\lambda,v_i))\cdot \bbm{1}(i \in S_\tau)\right)\right],
\end{aligned}
\end{eqnarray*}
where the inequality follows from \eqref{eq:lhs_mgf_ineq}. Similarly by conditioning with $\ep{F}_{\ell-1},  \cdots, \ep{F}_1$, we obtain, 
$$\bb{E} \left[\exp\left(\sum_{\tau = 1}^\ell (2\lambda\tilde{v}_{i,\tau} -2f(\lambda,v_i))\cdot \bbm{1}(i \in S_\tau)\right) \right] \leq 1.$$
From \eqref{eq:lhs_MI} and \eqref{eq:lhs_neg_corr}, we have
\begin{eqnarray*}
Pr\left(\sum_{\tau=1}^\ell \tilde{v}_{i,\tau} \bbm{1}(i \in S_\tau) > (1+\delta_i) v_i n_i(\ell)\right) \leq \bb{E}^{\frac{1}{2}}\left[\exp\Big(-2\lambda (1+\delta_i)v_i (1-f(\lambda,v_i)) n_i(\ell)\Big)\right].
\end{eqnarray*}
Therefore, we have 
\begin{eqnarray}\label{eq:one_sided_prelimn}
Pr\left(\sum_{\tau=1}^\ell \tilde{v}_{i,\tau} \bbm{1}(i \in S_\tau) > (1+\delta_i) v_i n_i(\ell)\right) \leq   \bb{E}^{\frac{1}{2}}\left[\underset{\lambda \in \Omega }{\min}\exp\Big(-2\lambda (1+\delta_i)v_i (1-f(\lambda,v_i)) n_i(\ell)\Big)\right], \;\;\;\;\;\;
\end{eqnarray}
where $\Omega = \{\lambda | 0 < \lambda < \frac{1}{2}\log{\frac{1+v_i}{v_i}}\}$ is the range of $\lambda$ for which the moment generating function in \eqref{eq:lhs_mgf_ineq} is well definred. Taking logarithm of the objective in \eqref{eq:one_sided_prelimn}, we have,
\begin{eqnarray}\label{convex_opt}
\underset{\lambda \in \Omega}{\text{argmin}}~e^{-2\lambda (1+\delta_i)v_i (1-f(\lambda,v_i)) \cdot n_i(\ell)} = \underset{\lambda \in \Omega}{\text{argmin}} -2(1+\delta_i) \lambda n_i(\ell) v_i  -  n_i(\ell)\log\left({1-v_i(e^{2\lambda} - 1)}\right).
\end{eqnarray}
Noting that the right hand side in the above equation is a convex function in $\lambda$, we obtain the optimal $\lambda$ by solving for the zero of the derivative. Specifically, at optimal $t$, we have 
$$e^{2\lambda} = \frac{(1+\delta_i)(1+v_i)}{1+v_i(1+\delta_i)}.$$
Substituting the above expression in \eqref{eq:one_sided_prelimn}, we obtain the following bound. 
\begin{eqnarray}\label{eq:one_sided_prelimn_bound}
Pr\left(\hat{v}_{i}(\ell) > (1+\delta_i) v_i \right) \leq \bb{E}^{\frac{1}{2}}\left[{\left(1 - \frac{\delta_i}{(1+\delta_i)(1+v_i)}\right)}^{n_i(\ell)v_i(1+\delta_i)} {\left(1+\frac{\delta_i v_i}{1+v_i}\right)^{n_i(\ell)}}\right].
\end{eqnarray}
For notational brevity, we will use $n$ to denote the random variable $n_i(\ell)$ and focus on bounding the right hand term in the above equation.

From Taylor series of $\log{(1-x)}$, we have that 
\begin{eqnarray*}\label{geq_second_order_del1}
nv_i(1+\delta_i)\log{\left(1- \frac{\delta_i}{(1+\delta_i)(1+v_i)} \right)} \leq -\frac{n\delta_i v_i}{1+v_i} - \frac{n\delta_i^2 v_i}{2(1+\delta_i)(1+v_i)^2},
\end{eqnarray*}
 From Taylor series for $\log{(1+x)}$, we have 
\begin{eqnarray*}\label{geq_first_order_del1}
n \log{\left(1+\frac{\delta_i v_i}{1+v_i}\right)} \leq \frac{n\delta_i v_i}{(1+v_i)}.
\end{eqnarray*}
Note that if $\delta_i > 1$, we can use the fact that $\log{(1+\delta_i x)} \leq \delta_i \log{(1+x)}$ to arrive at the preceding result.
Substituting the preceding two equations  in \eqref{eq:one_sided_prelimn_bound}, we have
\begin{eqnarray}\label{eq:geq_mul1}
Pr\left(\hat{v}_{i}(\ell) > (1+\delta_i) v_i \right) \leq 
\bb{E}^{\frac{1}{2}}\left[\exp{\left(-\frac{n \delta_i^2v_i}{2(1+\delta_i)(1+v_i)^2}\right)}\right].
\end{eqnarray}

\subsection*{ Bounding ${Pr\left(\hat{v}_{i}(\ell) < (1-\delta_i) v_i \right)}$:} Now to bound the other one sided inequality,  we use the fact that for any $\lambda >0$, $$\bb{E}\left(e^{-\lambda \tilde{v}_{i,\tau}\bbm{1}(i \in S_\tau)}\middle | \ep{F}_{\tau}\right) = \left(\frac{1}{1-v_i(e^{-\lambda}-1)}\right)^{\bbm{1}(i \in S_\tau)}.$$ and follow a similar approach. More specifically, from Markov Inequality,  for any $\lambda > 0$ and \mbox{$0 <\delta_i < 1,$}  we have 
\begin{eqnarray}\label{eq:rhs_MI}
\begin{aligned}
Pr\left(\sum_{\tau=1}^\ell \tilde{v}_{i,\tau} \bbm{1}(i \in S_\tau) < (1-\delta_i) v_i n_i(\ell)\right) &= Pr \left\{ \exp\left(-\lambda \sum_{\tau=1}^\ell \tilde{v}_{i,\tau} \bbm{1}(i \in S_\tau) \right)> \exp\left(-\lambda (1-\delta_i) v_i n_i(\ell)\right)\right\},\\
& = Pr \left\{ \exp\left(-\lambda \sum_{\tau=1}^\ell \tilde{v}_{i,\tau} \bbm{1}(i \in S_\tau) + \lambda (1-\delta_i) v_in_{i}(\ell)\right) > 1\right\}, \\
& \leq \bb{E} \left[\exp\left(-\lambda \sum_{\tau=1}^\ell \tilde{v}_{i,\tau} \bbm{1}(i \in S_\tau) + \lambda(1-\delta_i) v_in_i(\ell)\right)\right].
\end{aligned}
\end{eqnarray}
For notational brevity, denote $f(\lambda,v_i)$ by the function, 
$$f(\lambda,v_i) = -\frac{\log\left({1-v_i(e^{-2\lambda}-1)}\right)}{2}.$$
We have, 
\begin{eqnarray}\label{eq:rhs_neg_corr}
\begin{aligned}
&\bb{E} \left[\exp\left(-\lambda \sum_{\tau=1}^\ell \tilde{v}_{i,\tau} \bbm{1}(i \in S_\tau) + \lambda(1-\delta_i) v_in_i(\ell)\right)\right]\\
 &= \bb{E} \left[\exp\left(\sum_{\tau = 1}^\ell (-\lambda\tilde{v}_{i,\tau} -f(\lambda,v_i))\cdot \bbm{1}(i \in S_\tau)\right) \cdot \exp\Big(\lambda (1-\delta_i)v_i (1+f(\lambda,v_i)) n_i(\ell)\Big)\right],\\
&\leq \bb{E}^{\frac{1}{2}} \left[\exp\left(\sum_{\tau = 1}^\ell (-2\lambda\tilde{v}_{i,\tau} -2f(\lambda,v_i))\cdot \bbm{1}(i \in S_\tau)\right) \right] \cdot \bb{E}^{\frac{1}{2}}\left[\exp\Big(2\lambda (1-\delta_i)v_i (1+f(\lambda,v_i))  n_i(\ell)\Big)\right],
\end{aligned}
\end{eqnarray}
where the above inequality follows from Cauchy-Schwartz inequality. Let $\ep{F}_\tau$ be the filtration corresponding to the history until epoch $\tau$. Note that for any $\tau$, $\bbm{1}(i \in S_\tau)$ conditioned on $F_{\tau}$ is a constant and $\{\tilde{v}_{i,\tau} | \ep{F}_\tau\}$ is a geometric random variable. Therefore, for all $\tau \geq 1$ and for any $\lambda > 0,$ we have, $$\bb{E}\left(e^{-2\lambda \tilde{v}_{i,\tau}\bbm{1}(i \in S_\tau)}\middle | \ep{F}_{\tau}\right) = \left(\frac{1}{1-v_i(e^{-2\lambda}-1)}\right)^{\bbm{1}(i \in S_\tau)}.$$
Therefore, it follows that 
\begin{eqnarray}\label{eq:rhs_mgf_ineq}
\bb{E}\left(e^{(-2\lambda \tilde{v}_{i,\tau}- 2f(\lambda,v_i)) \cdot \bbm{1}(i \in S_\tau)}\middle | \ep{F}_{\tau}\right) \leq 1,
\end{eqnarray}
and 
\begin{eqnarray*}
\begin{aligned}
\hspace{-5mm}\bb{E} \left[\exp\left(\sum_{\tau = 1}^\ell (-2\lambda\tilde{v}_{i,\tau} -2f(\lambda,v_i))\cdot \bbm{1}(i \in S_\tau)\right) \right] &= \bb{E} \left[\bb{E}\left\{\exp\left(\sum_{\tau = 1}^\ell (-2\lambda\tilde{v}_{i,\tau} -2f(\lambda,v_i))\cdot \bbm{1}(i \in S_\tau)\right)  \middle | \ep{F}_\ell\right\}\right],\\
&\hspace{-40mm}= \bb{E} \left[\prod_{\tau=1}^{\ell-1}\exp\left((-2\lambda\tilde{v}_{i,\tau} -2f(\lambda,v_i))\cdot \bbm{1}(i \in S_\tau)\right)\cdot \bb{E}\left(e^{(-2\lambda \tilde{v}_{i,\ell}- 2f(\lambda,v_i)) \cdot \bbm{1}(i \in S_\ell)}\middle | \ep{F}_{\ell}\right)\right],\\
&\hspace{-35mm}=  \bb{E} \left[\prod_{\tau=1}^{\ell-1}\exp\left( (-2\lambda\tilde{v}_{i,\tau} -2f(\lambda,v_i))\cdot \bbm{1}(i \in S_\tau)\right)\right],
\end{aligned}
\end{eqnarray*}
where the inequality follows from \eqref{eq:rhs_mgf_ineq}. Similarly by conditioning with $\ep{F}_{\ell-1},  \cdots, \ep{F}_1$, we obtain, 
$$\bb{E} \left[\exp\left(\sum_{\tau = 1}^\ell (-2\lambda\tilde{v}_{i,\tau} -2f(\lambda,v_i))\cdot \bbm{1}(i \in S_\tau)\right) \right] \leq 1.$$
From \eqref{eq:rhs_MI} and \eqref{eq:rhs_neg_corr}, we have
\begin{eqnarray*}
Pr\left(\sum_{\tau=1}^\ell \tilde{v}_{i,\tau} \bbm{1}(i \in S_\tau) < (1-\delta_i) v_i n_i(\ell)\right) \leq \bb{E}^{\frac{1}{2}}\left[\exp\Big(2\lambda (1-\delta_i)v_i (1+f(\lambda,v_i))  n_i(\ell)\Big)\right].
\end{eqnarray*}
Therefore, we have 
\begin{eqnarray*}
Pr\left(\hat{v}_i (\ell)< (1-\delta_i) v_i \right) \leq  \bb{E}^{\frac{1}{2}}\left[\underset{ \lambda > 0}  {\min}\exp\Big(2\lambda (1-\delta_i)v_i (1+f(\lambda,v_i))  n_i(\ell)\Big)\right].
\end{eqnarray*}
Following similar approach as in optimizing the previous bound (see \eqref{eq:one_sided_prelimn}) to establish the following result. For notational brevity, we will use $n$ to denote the random variable $n_i(\ell)$.
\begin{eqnarray*}
Pr\left(\hat{v}_i(\ell) < (1-\delta_i) v_i\right) \leq \bb{E}^{\frac{1}{2}}\left[{\left(1 + \frac{\delta_i}{(1-\delta_i)(1+v_i)}\right)}^{nv_i(1-\delta_i)} {\left(1-\frac{\delta_i v_i}{1+v_i}\right)^n}\right].
\end{eqnarray*}
Now we will use Taylor series for $\log{(1+x)}$ and $\log{(1-x)}$ in a similar manner as described for the other bound to obtain the required result. In particular, since $1-\delta_i \leq 1$, we have for any $x > 0$ it follows that $(1+\frac{ x}{1-\delta_i})^{(1-\delta_i)} \leq (1+x)$ . Therefore, we have
\begin{eqnarray}\label{eq:one_sided_prelimn_bound_leq}
Pr\left(\hat{v}_i(\ell) < (1-\delta_i) v_i\right) \leq \bb{E}^{\frac{1}{2}}\left[{\left(1 + \frac{\delta_i}{1+v_i}\right)}^{nv_i} {\left(1-\frac{\delta_i v_i}{1+v_i}\right)^n}\right].
\end{eqnarray}

Note that since $\tilde{v}_{i,\tau} \geq 0$ for all $i,\tau$, we have a zero probability event if $\delta_i > 1$. Therefore, without loss of generality, we assume $\delta_i < 1$ and from Taylor series for $\log{(1-x)}$, we have 
\begin{eqnarray*}\label{leq_first_order_del1}
n\log{\left(1- \frac{\delta_iv_i}{1+v_i} \right)} \leq -\frac{n\delta_i v_i}{1+v_i}, 
\end{eqnarray*}
 and from Taylor series for $\log{(1+x)}$, we have 
\begin{eqnarray*}
n \log{\left(1+\frac{\delta_i v_i}{1+v_i}\right)} \leq \frac{n\delta_i }{(1+v_i)} - \frac{n \delta_i^2 v_i}{6 (1+v_i)^2}\left(3 - \frac{2\delta_i v_i}{1+v_i}\right).
\end{eqnarray*}
Therefore, substituting the preceding equations  in \eqref{eq:one_sided_prelimn_bound_leq}, we have, 
\begin{eqnarray}\label{eq:leq_mul1}
Pr\left(\hat{v}_i < (1-\delta_i) v_i\right) \leq 
\exp{\left(- \frac{n \delta_i^2 v_i}{6 (1+v_i)^2}\left(3 - \frac{2\delta_i\mu}{1+v_i}\right)\right)}.
\end{eqnarray}

The result follows from \eqref{eq:geq_mul1} and  \eqref{eq:leq_mul1}.   \hfill $\Halmos$
\end{proof}

\subsection*{Proof of Lemma \ref{multiplicative_chernoff_estimates}.}

Let $\delta_i =  \sqrt{\frac{4 (v_i+2)m\log{(\rho+1)}}{v_i n_i(\ell)}} .$ We analyze the cases $\delta_i \leq \frac{1}{2}$ and $\delta_i \geq \frac{1}{2}$ separately. 

\medskip \noindent {\bf Case 1: $\delta_i \leq \frac{1}{2}:$}  For any $v_i \leq 1$ and $\delta_i \leq 1/2$, we have,$$\frac{v_i \delta_i^2n_i(\ell)}{2(1+\delta_i)(1+v_i)^2} \geq \frac{v_i \delta_i^2n_i(\ell)}{6(1+v_i)} \geq {m\log{(\rho+1)}},$$
and
$$\frac{v_i\delta_i^2 n_i(\ell)}{6 (1+v_i)^2}\left(3 - \frac{2\delta_iv_i}{1+v_i}\right) \geq \frac{v_i \delta_i^2n_i(\ell)}{6(1+v_i)} \geq {m\log{(\rho+1)}}.$$
Therefore, substituting $\delta_i =  \sqrt{\frac{4 (v_i+2)m\log{(\rho+1)}}{v_i n_i(\ell)}} $ in Theorem \ref{multiplicative_chernoff_geometric} with $\delta_i$,  we have, 
\begin{eqnarray}
\begin{aligned}
\ep{P}\left(2\hat{v}_{i}(\ell) \geq {v_i}\right) &\geq  1-\frac{1}{\rho^m},\\
 \ep{P}\left(\left|\hat{v}_{i}(\ell) - {v_i}\right| <  \sqrt{\frac{4v_i (v_i+2)m\log{(\rho+1)}}{ n_i(\ell)}}\right) &\geq  1-\frac{2}{\rho^m}.
\end{aligned}
\end{eqnarray}
From the above three results, we have,
\begin{eqnarray}\label{ineq_part_2_mu_l1}
\ep{P}\left(\left|\hat{v}_{i}(\ell) - {v_i}\right| < \sqrt{\frac{16\hat{v}_{i}(\ell)\left(\hat{v}_{i}(\ell)+1\right)\log{(\rho+1)}}{n_i(\ell)}}\right) \geq \ep{P}\left(\left|\hat{v}_{i}(\ell) - {v_i}\right| < \sqrt{\frac{4v_i(v_i+2) \log{(\rho+1)}}{n_i(\ell)}}\right) \geq  1-\frac{3}{\rho^m}. \;\;\;\;\;\;\;\;\;\;
\end{eqnarray}
By assumption, $v_i \leq 1$. Therefore, we have $v_i(v_i+2) \leq 3v_i$ and, 
\begin{eqnarray*}
 \ep{P}\left(\left|\hat{v}_{i}(\ell) - {v_i}\right| < \sqrt{\frac{12v_i \log{(\rho+1)}}{n_i(\ell)}}\right) \geq  1-\frac{3}{\rho^m}.
\end{eqnarray*}

\medskip \noindent {\bf Case 2: ${\delta}_i > \frac{1}{2}:$} Now consider the scenario, when $\sqrt{\frac{4 (v_i+2)m\log{(\rho+1)}}{v_i n_i(\ell)}} > \frac{1}{2}$. Then, we have, 
$$\bar{\delta}_{i} \overset{\Delta}{=} {\frac{8(v_i+2)m \log{(\rho+1)}}{v_i n_i(\ell)}}  \geq \frac{1}{2},$$ which implies for any $v_i \leq 1$,
\begin{eqnarray*}
\begin{aligned}
\frac{nv_i \bar{\delta}_i^2}{2(1+\bar{\delta}_i)(1+v_i)^2} & \geq \frac{nv_i \bar{\delta}_i}{12(1+v_i)},\\
\frac{n \bar{\delta}_i^2 v_i}{6 (1+v_i)^2}\left(3 - \frac{2\bar{\delta}_iv_i}{1+v_i}\right) & \geq \frac{nv_i \bar{\delta}_i}{12(1+v_i)}.
\end{aligned}
\end{eqnarray*}
Therefore, substituting the value of $\bar{\delta}_i$ in Theorem \ref{multiplicative_chernoff_geometric}, we have 
\begin{eqnarray*}
\ep{P}\left(\left|\hat{v}_{i}(\ell) - v_i\right|> \frac{24m\log{(\rho+1)}}{n}\right) \leq \frac{2}{\rho^m}. 
\end{eqnarray*}


\end{APPENDICES}
\section*{Acknowledgments.}
V. Goyal is supported in part by NSF Grants CMMI-1351838 (CAREER) and CMMI-1636046. A. Zeevi is supported in part by NSF Grants NetSE-0964170 and  BSF-2010466.

%


%
%
%






\bibliographystyle{}
\bibliography{acmsmall-sample-bibfile}

\end{document}